\documentclass[twoside]{article}

\usepackage[accepted]{aistats2025}
\usepackage[round]{natbib}

\usepackage{verbatim}
\usepackage{hyperref}
\hypersetup{
    colorlinks=true,
    linkcolor=blue,
    filecolor=magenta,      
    urlcolor=red,
    citecolor=blue,
    pdftitle={Overleaf Example},
    pdfpagemode=FullScreen,
    }
\usepackage{url}
\usepackage{wrapfig}
\usepackage{soul}
\usepackage{graphicx}
\usepackage{amsmath}
\usepackage{amssymb}
\usepackage{amsthm}
\usepackage{mathtools}
\usepackage{array}
\usepackage{algorithm}
\usepackage{algorithmic}
\usepackage{dsfont} 
\usepackage{enumitem}
\usepackage[textsize=small]{todonotes}
\usepackage{selectp}
\usepackage{booktabs}
\usepackage{tabularx}
\usepackage{svg}            
\usepackage[algo2e]{algorithm2e}
\usepackage[hang,flushmargin]{footmisc}

\newcolumntype{L}[1]{>{\raggedright\let\newline\\\arraybackslash\hspace{0pt}}m{#1}}
\newcolumntype{C}[1]{>{\centering\let\newline\\\arraybackslash\hspace{0pt}}m{#1}}
\newcolumntype{R}[1]{>{\raggedleft\let\newline\\\arraybackslash\hspace{0pt}}m{#1}}
\newcolumntype{x}[1]{>{\centering\let\newline\\\arraybackslash\hspace{0pt}}p{#1}}

\DeclareMathOperator*{\argmax}{arg\,max}

\newcommand{\inputspace}{\mathcal{X}}

\newcommand{\explainer}{E}
\newcommand{\expdistnfull}{\mu_{(E, F, X)}}

\newcommand{\explanationSpace}{\mathcal{E}}
\newcommand{\distnset}{\mathcal{P}(\explanationSpace)}
\newcommand{\centermu}{\eta_{\mathcal{E}\#}\mu}
\newcommand{\centermuclean}{\bar \mu}
\newcommand{\refdist}{U_{\mathcal{E}}}
\newcommand{\refdistemp}{U_{\mathcal{E}}^{(N)}}
\newcommand{\reals}{\mathbb{R}}

\newcommand{\abbr}{WG}
\newcommand\inv[1]{#1\raisebox{1.15ex}{$\scriptscriptstyle-\!1$}}
\theoremstyle{plain}
\newtheorem{theorem}{Theorem}

\newtheorem{lemma}[theorem]{Lemma}

\theoremstyle{definition}
\newtheorem{definition}[theorem]{Definition}

\theoremstyle{remark}

\begin{document}

%

%
\runningauthor{Davin Hill$^*$, Josh Bone$^*$, Aria Masoomi, Max Torop, Jennifer Dy}

\twocolumn[

\aistatstitle{Axiomatic Explainer Globalness via Optimal Transport}

\aistatsauthor{ Davin Hill$^*$ \And Josh Bone$^*$ \And Aria Masoomi}
\aistatsaddress{Northeastern University \\ dhill@ece.neu.edu \And Northeastern University \\  bone.j@northeastern.edu \And Northeastern University \\ masoomi.a@northeastern.edu}

\aistatsauthor{ Max Torop \And Jennifer Dy}
\aistatsaddress{Northeastern University \\ torop.m@northeastern.edu \And Northeastern University \\ jdy@ece.neu.edu} ]

\begin{abstract}
Explainability methods are often challenging to evaluate and compare.
With a multitude of explainers available, practitioners must often compare and select explainers based on quantitative evaluation metrics.
One particular differentiator between explainers is the diversity of explanations for a given dataset; i.e. whether all explanations are identical, unique and uniformly distributed, or somewhere between these two extremes.
In this work, we define a complexity measure for explainers, \emph{globalness}, which enables deeper understanding of the distribution of explanations produced by feature attribution and feature selection methods for a given dataset.
We establish the axiomatic properties that any such measure should possess and prove that our proposed measure, \emph{Wasserstein Globalness}, meets these criteria.
We validate the utility of Wasserstein Globalness using image, tabular, and synthetic datasets, empirically showing that it both facilitates meaningful comparison between explainers and improves the selection process for explainability methods.
\end{abstract}
\section{INTRODUCTION}
\vspace{-2mm}

Machine Learning models have become increasingly prevalent across various domains \citep{abiodun, lecun2015}, however, this increased performance is often associated with higher levels of complexity and opaqueness.
The sub-field of post-hoc Explainable AI (XAI) methods has emerged to address the challenge by providing explanations corresponding to black-box model predictions, clarifying the reasoning behind them.
In this work we focus on feature-based methods, or \emph{explainers}, that either provide real-valued (feature attribution) or binary (feature selection) scores as an \emph{explanation} for a given data sample.

A wide variety of explainers have been introduced in the recent years, prompting a need for quantitative metrics to compare explainer performance.
A challenge in evaluating explainers is that there is typically no ``ground truth'' for calculating traditional performance metrics, such as accuracy.
Instead, recent works have proposed analogous metrics, which we refer to as \emph{faithfulness} metrics \citep{chen2022makes}, that quantify the ability of generated explanations to reflect the predictive behavior of the black-box models (e.g. \citet{yeh2019fidelity,ROAR,ROAD}).

\begin{figure*}[t]
  \centering
  \includegraphics[width=0.95\linewidth]{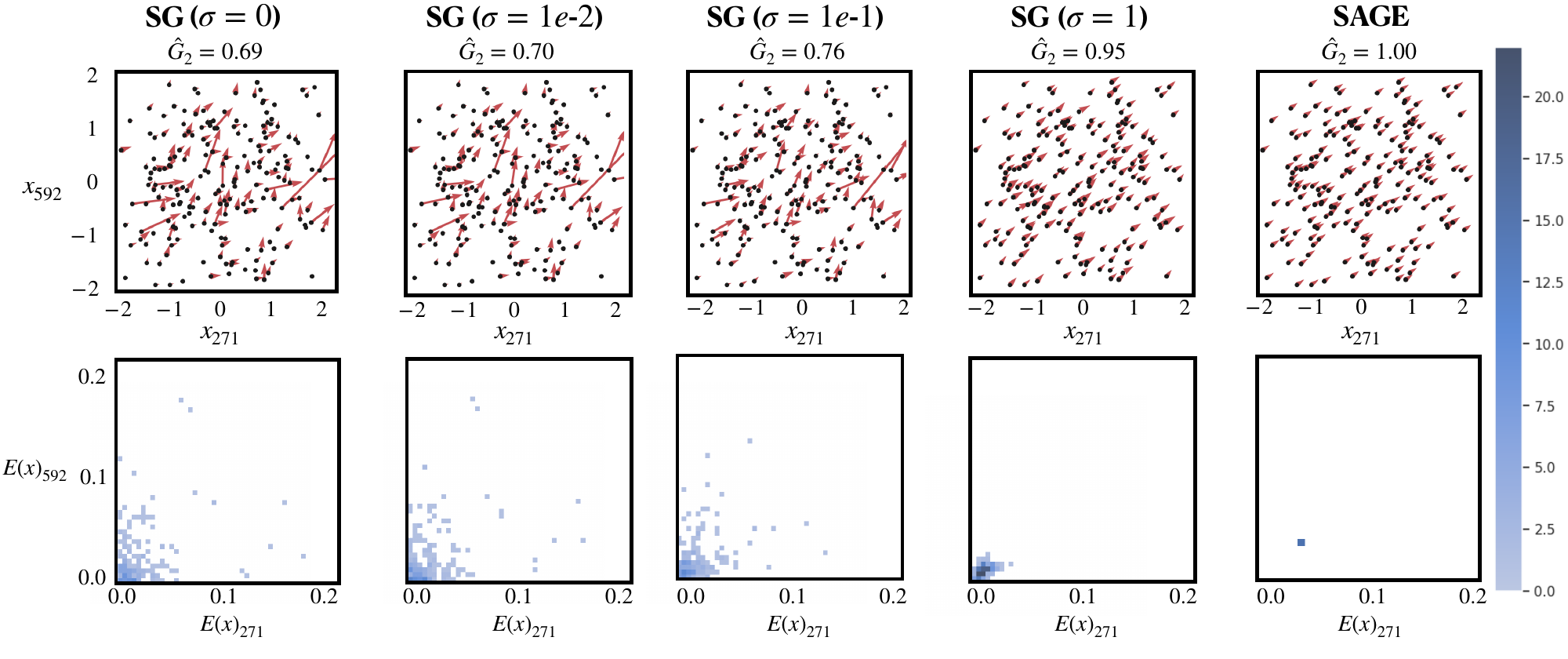}
  \vspace{-5mm}
  \caption{
  Different explainers and hyperparameters can exhibit different levels of \emph{globalness} for the same black-box model and dataset. We plot SmoothGrad (SG; $\sigma$ indicates smoothing, see Eq. \ref{eq:smoothgrad}) and SAGE explanations for CIFAR10, projected to 2-d for visualization.
  \textbf{Top:} Samples are plotted in black; attributions are shown as red vectors originating from their respective sample.
  \textbf{Bottom:} 2-d histograms of the distribution of explanations.
  As $\sigma$ increases for SG, the histograms condense towards the dirac, indicating higher globalness.
  SAGE produces the same explanation for every sample.
  Wasserstein Globalness ($\hat{G}_2$), denoted at the top of each column, captures this spectrum of globalness. 
  }
  \label{fig:example1}
\end{figure*}

However, faithfulness metrics alone are insufficient for comparing explainers.
In a practical scenario, a user may be given a trained prediction model and dataset, and is deciding among a set of explainers to use.
While faithfulness performance is a critical deciding factor, there are often cases where different explainers, or different sets of explainer hyperparameters, yield similar faithfulness scores.
This is analagous to a model selection problem: when two models achieve similar accuracy, the simpler model is generally preferred \citep{vapnik2013nature}.
Therefore, when given explainers with \emph{similar faithfulness scores}, it is important to also consider explainer complexity, and, all else being equal, select the explainer with lowest complexity.

In this work we investigate an explainer complexity metric based on the diversity of explanations generated over a given dataset and black-box model.
In particular, if we consider explainers as function approximators of black-box models \citep{han2022explanation, hill2024boundary}, then explainers that produce more diverse explanations for the same dataset are inherently more complex function approximators. 
We refer to this complexity metric as \emph{globalness}, so-named because a global explainer (i.e. an explainer that generates the same explanations over the entire dataset) represents a natural upper bound to a globalness metric.
Intuitively, an explainer with a higher globalness score generates fewer distinct explanations, identifies more globally-invariant features, and/or produces more similar explanations (Figure \ref{fig:example1}). 
Therefore, each explanation from a high-globalness explainer has greater generalizability in terms of the number of data samples it ``explains''.
A similar concept has been explored in cognitive science as a desirable property for explanations, referred to as ``breadth'' \citep{lombrozo2016explanatory} or ``degree of integration'' \citep{ylikoski2010dissecting}.

To formally define this globalness measure, we first introduce six desirable properties that such a metric should satisfy. Building on these axioms, we propose a novel method, Wasserstein Globalness (\abbr{}), which quantifies explainer globalness by evaluating the diversity of explanations for a given dataset and black-box model.
\abbr{} is a continuous, flexible metric that adheres to the desired axiomatic properties, and is applicable to both discrete and continuous explanations.
Notably, \abbr{} also enables users to choose a distance metric for explanations, allowing customization of how similarity is measured and specifying the transformations which \abbr{} remains invariant to.
To the best of our knowledge, we are the first to formalize a globalness measure for comparing explainers.
In summary, our main contributions are:
\setlist{nolistsep}
\begin{itemize} [leftmargin=0.5cm]
    \item  We motivate the need for a measure of explanation globalness and introduce axiomatic properties which any reasonable measure should satisfy.
    \item We propose {\em Wasserstein Globalness} (\abbr{}), which we prove satisfies the aforementioned axioms.
    \item We provide theoretical results characterizing the sample complexity for \abbr{}.
    \item Experiments empirically validate the ability for \abbr{} to better compare different explainers. 
\end{itemize}

\section{RELATED WORK} \label{sec:related_work}
\vspace{-2mm}

A wide variety of methods have been introduced to explain black-box models \citep{guidottiSurveyMethodsExplaining2018, zhang_survey}.
Therefore, quantitative evaluation metrics have emerged as a key research area \citep{nauta2023anecdotal, chen2022makes}. Specifically, we focus on metrics for post-hoc feature attribution and feature selection methods, which can be categorized into \emph{faithfulness}, \emph{robustness}, and \emph{complexity} metrics.

\emph{Faithfulness} metrics evaluate how well explanations reflect the true underlying behavior of the black-box model. These methods often involve perturbing \citep{yeh2019fidelity, samek2016evaluating} or masking \citep{petsiukRISERandomizedInput2018, ROAR, ROAD, chen2018learning, shah2021input, dabkowski_gal, jethani2022fastshap, arya2019one} features with the highest (or lowest) attribution values and measuring the change in model output.
Other works focus on different aspects of \emph{robustness}, such as explanation stability \citep{bansalSAMSensitivityAttribution2020}, adversarial robustness \citep{wang2020smoothed, dombrowski_explanations_can_be_manipulated, ghorbaniInterpretationNeuralNetworks2019,rieger2020simple}, sensitivity to perturbations \citep{alvarez-melisRobustnessInterpretabilityMethods2018, zulqarnain_lipschitzness}, or sensitivity to distribution shift \citep{lakkaraju2020robust, upadhyay2021towards}.

Our work is most related to the property of \emph{complexity}.
There have been various perspectives on what constitutes low-complexity explanations. Some works quantify the spread of feature attribution values \emph{within each explanation} \citep{nguyen2020quantitative, bhattEvaluatingAggregatingFeaturebased2020, chalasani2020concise}.
\citet{samek2016evaluating} evaluate explanation complexity using the filesize of the compressed saliency maps.
\citet{nguyen2020quantitative} quantify an explanation's broadness by measuring dependency between the data distribution and an interpretable feature mapping. \citet{zhang_survey} loosely define a local, semi-local, and global taxonomy for explainers.
In contrast, we propose a complexity metric that quantifies the diversity, or spread, of \emph{a distribution of explanations} over a dataset.

The proposed method utilizes the Wasserstein metric \citep{kantorovich}, which has been applied in various ways related to explainers.
\citet{chaudhury2024explainable} uses Wasserstein distance to measure the distance between the data distribution and predictive distribution of a black-box model.
\citet{fel2022good} investigate the use of Wasserstein distance as a pairwise similarity measure between explanations.

\section{AXIOMATIC GLOBALNESS} \label{sec:axioms}
\vspace{-2mm}

To guide us with defining an appropriate measure of globalness, we will first discuss the desired properties which should be satisfied by \textit{any} measure of explanation globalness (Section \ref{sec:properties}).
In Section \ref{sec:candidates}, we briefly review why other reasonable candidates fall short, in light of our axioms. 
In Section \ref{sec:wasserstein_globalness}, we define Wasserstein Globalness, and show that it satisfies each axiom.

\subsection{Technical Preliminaries}
\vspace{-1mm}

Consider a black-box prediction model $F : \inputspace \rightarrow \mathcal{Y}$, where $\inputspace \subseteq \reals^d$ is the space of model inputs, $\mathcal{Y} \subseteq \reals^c$ is the space of labels, $d$ is the dimension of the data, and $c$ is the number of class labels. 
An \textit{explainer} for this model is a function $\explainer(F, \cdot) : \inputspace \rightarrow \explanationSpace$, with outputs $\explainer(F, x)$, $x \in \inputspace$ which are referred to as \textit{explanations}. 
The space $\explanationSpace$ is either a subset of $\reals^s$ for feature attribution explainers, or the vertices of a hypercube, $\{0,1\}^s$, for feature selection explainers, where $s$ represents the dimension of the explanations\footnote{Note that $s$ does not necessarily equal $d$ (e.g. when using a low-dimensional mapping \citep{ribeiroLIME}).}.
We will use $\distnset$ and $\mathcal{P}(\inputspace)$ to denote the space of Borel probability measures on $\explanationSpace$ and $\inputspace$, respectively.
Given an explainer $E$, a model $F$, and data distribution $X \in \mathcal{P}(\inputspace)$, we can define $\expdistnfull \in \distnset$ as the corresponding push-forward distribution of explanations generated by applying the explainer $E$ over $X$. We assume that the support of $\expdistnfull$ is a bounded subset of $\explanationSpace$.\footnote{This is a mild assumption; in practice the user has a finite number of explanations sampled from $\expdistnfull$.}

Additionally, let $d_{\explanationSpace}$ be a distance metric between explanations; the specific distance metric $d_{\explanationSpace}$ should be selected based on the explanation space $\mathcal{E}$.
We define two options: $d_A(x,y) = ||x-y||_2$ for feature attribution and $d_S(x,y) = \sum_i \mathds{1}(x_i \neq y_i)$ for feature selection. Selection of $d_{\explanationSpace}$ is further discussed in Section \ref{sec:wasserstein_globalness}.
For any $P,Q \in \distnset$, $\Pi(P, Q)$ denotes the set of joint probability measures in $\mathcal{P}(\mathcal{E} \times \mathcal{E})$ which marginalize to $P$ and $Q$.
We denote the Dirac measure centered at $x$ as $\delta_{x}$, which takes value $1$ at $x$ and $0$ elsewhere.
For a measurable map $\phi: \mathcal{E}_1 \rightarrow \mathcal{E}_2$ between any $\mathcal{E}_1, \mathcal{E}_2 \subseteq \mathcal{E}$, we denote the \textit{push-forward} of a measure $\nu$ by $\phi$ as $\phi_{\#} \nu(A) = \nu(\phi^{-1}(A)) \; \; \forall A \subseteq \mathcal{E}_2$.
We provide a summary of notations in Appendix \ref{app:background} Table \ref{tab:notation}.

\subsection{Desired Globalness Properties}\label{sec:properties}
\vspace{-1mm}

Given a black-box model $F$ and data distribution $X$, we define the globalness of an explainer $E$ as the spread of the distribution of explanations $\expdistnfull$\footnote{We omit the subscript $(E,F,X)$ when clear from context.} over $\explanationSpace$. 
In this section, we introduce properties which are intuitively desirable for any globalness measure.

First, we establish basic properties that we expect a globalness metric should hold (\ref{prop:non-negativity}, \ref{prop:continuity}, \ref{prop:convexity}).
In \ref{prop:local} and \ref{prop:global}, we define the desired upper and lower bounds. Then, in \ref{prop:isometry_invariance} we define how a globalness measure should capture behavior related to explanation distances and transformations.
In lieu of specific measures, which we define in Section \ref{sec:wasserstein_globalness}, we reference a general globalness measure $G: \distnset \to \mathbb{R}$, which maps a distribution of explanations to a real-valued globalness score.

\begin{enumerate}[label=P\arabic*, font=\textbf, wide=0pt]
    
  \item  \label{prop:non-negativity} \textbf{(Non-negativity)} \emph{For any probability measure $\mu \in \distnset$, $G(\mu) \geq 0$.}
  
  \item \label{prop:continuity} \textbf{(Continuity)} \emph{Let $\{\mu^{(n)}\}$ be a sequence of probability measures which converges weakly to $\mu$. Then $G(\mu^{(n)}) \rightarrow G   (\mu)$.}

  \item \label{prop:convexity} \textbf{(Convexity)} \emph{Let $\nu = \lambda P + (1 - \lambda) Q$ for some $\lambda \in [0,1]$ and $P,Q \in \distnset$. Then $G(\nu) \leq \lambda G(P) + (1 - \lambda) G(Q)$.}
  
\vspace{10pt}
\ref{prop:continuity} ensures that $G$ is continuous with respect to the explanation distribution $\mu$, even if $\mu$ is empirically estimated. 
We further discuss $\mu$ approximation in Section \ref{sec:sample_complexity}.
In \ref{prop:convexity} we consider the property of \emph{convexity}, which implies that the globalness of a mixture of two explanation distributions is less than the average of each distribution's globalness.
Convexity is a common qualification of diversity-related measures in various fields  \citep{ricotta}.

\vspace{6pt}

\textbf{Upper and lower bounds of $\boldsymbol{G}$.}
We next establish the desired bounds for $G$.
A global explainer generates a single explanation for all samples $x \in \mathcal{X}$; we define this as the theoretical upper bound for quantifying explainer globalness, as it minimizes the number of explanations for a given dataset.
Importantly, this is not a prescriptive expectation for explainer behavior across all domains but rather a theoretical upper bound since a global explainer can always be applied.
This property is formalized in \ref{prop:global}.

\vspace{6pt}
Conversely, we need to define a baseline distribution, representing the minimally-global distribution of explanations, to establish a lower bound for globalness.
Similar to $d_\mathcal{E}$, this baseline distribution, denoted $\refdist \in \distnset$, should be selected based on the application or type of explanations.
Intuitively, an explainer with minimum globalness should have different explanations $E(x)$ for each sample $x \in \mathcal{X}$ and have ``evenly'' distributed explanations over $\mathcal{E}$.
Therefore a natural choice for this baseline is a uniform distribution over $\mathcal{E}$.
We denote $U_A$ for the feature attribution case and $U_S$ for the feature selection case with corresponding density/mass functions $u_A$ and $u_S$, respectively:
\begin{equation} \label{eq:uniform} \nonumber
\underbrace{u_A(x;k) \propto  \begin{cases}1 & ||x||_2 \leq k \\ 0 & \textrm{otherwise} \end{cases} }_{\textrm{Feature Attribution: }\mathcal{E} \subseteq \mathbb{R}^s}
\;
\underbrace{u_S(x) \propto  \begin{cases}1 & x \in \{0,1\}^s \\ 0 & \textrm{otherwise} \end{cases}}_{\textrm{Feature Selection: }\mathcal{E} \subseteq \{0,1\}^s}
\end{equation}

In the feature attribution case, since $\mathcal{E}$ is not necessarily bounded, we define the support of $U_A$ to be a closed ball of dimension $s$ and radius $k$ \footnote{For conciseness, we use $U_{A}$ in place of $U_{A,k}$ with $k$ fixed using the method as described.}.
The radius $k$ determines the baseline reference measure for comparing explanation globalness, and therefore should be selected such that $\textrm{supp}(\mu) \subseteq \textrm{supp}(U_A)$, where $\textrm{supp}(\mu)$ indicates the support of measure $\mu$.
In the use-case of comparing explainers, one option is to estimate this value empirically: specifically, let $\{E_1,...,E_m\}$ be a set of explainers with corresponding probability measures $\textrm{M} = \{\mu_{(E_1, F, X)},...,\mu_{(E_m, F, X)}\}$.
We fix $k = \max_{\zeta \in \textrm{M}} [ \sup_{x \in \textrm{supp}(\zeta)} ||x||_2]$.

\vspace{6pt}

In addition, to ensure consistent comparison between each $\{\mu_{(E_1, F, X)},...,\mu_{(E_m, F, X)}\}$ and $\refdist$, we define a translation to center each measure.
Let $\eta_\mathcal{E}: \mathcal{E} \rightarrow \mathcal{E}$ be a mapping to be defined based on $\mathcal{E}$.
In the feature attribution case, we define $\eta_A$ such that the push-forward of a probability measure $P$ has mean zero, i.e. $\int_{\mathcal{E}} x \; d(\eta_{A\#}P) = 0$. 
The feature selection case does not require centering; we set $\eta_S$ to be the identity mapping for notational consistency and define $\eta_{\mathcal{E}}$ to indicate either $\eta_{A}$ or $\eta_{S}$ depending on the relevant framework.

\vspace{6pt}

Together, \ref{prop:local} and \ref{prop:global} fix minimally global explanations (defined to be $\refdist$) to achieve zero globalness, and fully-global explanations to achieve maximum globalness.

\vspace{6pt}
\item \label{prop:local} \textbf{(Fully-local measure)}. \emph{Let $\refdist \in \distnset$ be a uniform measure with density defined in Eq. \ref{eq:uniform}. Then $\centermu = \refdist \iff G(\mu) = 0$.}
\item  \label{prop:global} \textbf{(Fully-global measure)} \emph{There exists an $x_0 \in \explanationSpace$ for which $G_p(\delta_{x_0}) \geq G_p(\mu)$ for all $\mu \in \distnset$.}

\vspace{10pt}

  \begin{figure*}[t]
  \centering
  \includegraphics[width=0.95\linewidth]{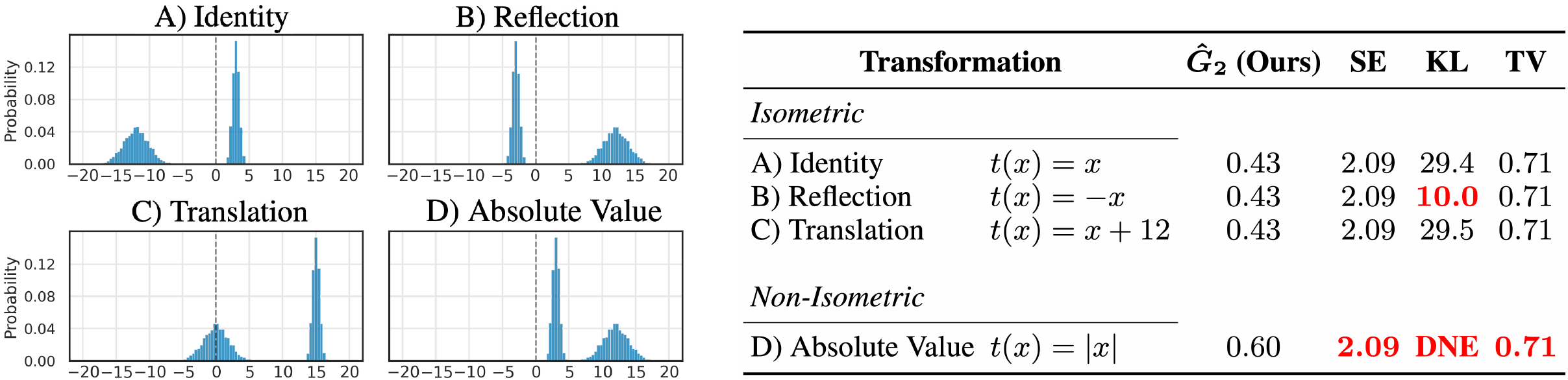}
  \caption{\textbf{(Left)} Example of transformations applied to a synthetic distribution of 1d explanations, plotted as histograms. 
  Transformations B) and C) are isometric; D) is non-isometric.
  \textbf{(Right)} We compare \abbr{} ($\hat G_2$) to Entropy (SE) and two f-Divergences (Kullback-Leibler (KL) and Total Variation (TV)).
  KL is undefined for D) due to the change in support after transformation.
  \ref{prop:isometry_invariance} states that a globalness measure should be invariant to isometries while being sensitive to non-isometries.
  Violations of \ref{prop:isometry_invariance} are highlighted (bold, red).
  Only \abbr{} fully satisfies \ref{prop:isometry_invariance}.
  Figure details are provided in App. \ref{app:properties_example}.
  }
  \label{fig:properties_example}
\end{figure*}

\textbf{Sensitivity to explanation distance metric $\boldsymbol{d_\mathcal{E}}$.}
The final property relates to the geometry of the distribution of explanations.
Intuitively, globalness measures the ``spread'' of explanations, which can be quantified by the distances between explanations.
Transformations that \emph{preserve} pairwise distances (i.e. isometries) therefore should not affect a globalness metric. Conversely, $G$ should be sensitive to transformations that scale or otherwise change pairwise distances.
\vspace{6pt}
  \item \label{prop:isometry_invariance} \textbf{(Selective Invariance to Transformations)} \emph{Let $T_{\explanationSpace, d_{\explanationSpace}}$ be the group of distance-preserving measurable maps $\phi$.}
  \begin{align*} \nonumber
      T_{\explanationSpace, d} = \{\phi : \explanationSpace \rightarrow \explanationSpace | d_{\explanationSpace}(x, y) = d_{\explanationSpace}(\phi(x), \phi(y)) \forall x,y \in \explanationSpace\}
  \end{align*} 
  \emph{Then $G$ is invariant with respect to the group $T_{\explanationSpace, d_{\explanationSpace}}$:}
  \begin{equation} \label{eq:isom_invariance}
    G(\mu) = G(\phi_{\#} \mu) \ \text{for all} \ \phi \in T_{\explanationSpace, d_{\explanationSpace}}
  \end{equation}
  \emph{Additionally, let $S$ denote all measurable bijective maps $\psi : \explanationSpace \rightarrow \explanationSpace$. Then $G$ is \emph{not} invariant to $S$:}
\begin{equation}
    \exists \psi \in S  \ \text{for which} \ G(\mu) \neq G(\psi_{\#} \mu)
\end{equation}
An example of isometric and non-isometric transformations are shown in Figure \ref{fig:properties_example} (Left) and App. \ref{app:additional_isometry_figure}. A measure of globalness should take into account a notion of similarity between explanations, which is captured by the selected distance metric, $d_{\explanationSpace}$. 
Distance or similarity metric selection for explanations is an open area of research; various metrics have been used in the literature \citep{fel2022good,ghorbaniInterpretationNeuralNetworks2019}, such as Spearman's rank correlation, Euclidean distance, or top-k intersection.
The choice of metric should be made based on the specific application, therefore it is important for any globalness metric to allow for user flexibility in selecting $d_{\explanationSpace}$.

\vspace{6pt}

Therefore, \ref{prop:isometry_invariance} ensures that the globalness metric is sensitive to transformations that affect pairwise explanation distances, but conversely is invariant to distance-preserving transformations.
This notion of distance and similarity is directly encoded by the choice of $d_{\explanationSpace}$. 

\end{enumerate}

\subsection{Candidate Measures} \label{sec:candidates}
\vspace{-1mm}
Several traditional metrics can be trivially used to estimate $G$; however, these metrics do not fully satisfy the stated properties.

\textbf{Shannon Entropy.}
Entropy \citep{shannon} is often used to quantify the ``spread'' of a distribution.
Let $\nu, u$ represent the density/mass functions associated with measures $\mu$ and $\refdist$, respectively. Then Shannon Entropy is defined as $H(X) = \mathbb{E}_{X \sim \nu}\left[-\log\left(\nu(X)\right)\right]$.

\textbf{f-Divergences.}
One might also consider using an f-divergence \citep{fdiv}
between $\mu$ and the baseline measure $\refdist$: $D_f(\nu\|u) = \mathbb{E}_{X \sim u}\left[f\left(\frac{\nu(X)}{u(X)}\right)\right]$.

While both Entropy and f-divergences satisfy several of the desiderata in \ref{sec:properties}, they both fail to satisfy \ref{prop:isometry_invariance}. We provide an empirical example in Figure \ref{fig:properties_example} and prove this theoretically in App. \ref{sec:pf_fdiv}.
Intuitively, both approaches fail to capture explanation distances.
For example, both candidates are invariant to feature permutation, which is a non-isometric transformation.

\section{WASSERSTEIN GLOBALNESS} \label{sec:wasserstein_globalness}
\vspace{-2mm}
We now formalize a globalness metric to satisfy the desired properties outlined in Section \ref{sec:axioms}.
In \ref{prop:local}, we established a reference distribution, $\refdist \in \distnset$, that represented the minimally global explanation on $\mathcal{E}$.
From this definition, we can apply a distance metric on $\distnset$ to measure the distance between the explanation distribution and the reference distribution.
In particular, we use the p-Wasserstein distance \citep{kantorovich}, which allows the \abbr{} formulation to inherit several useful properties.

\begin{definition}[Wasserstein Globalness] \emph{Let $(\explanationSpace, d_{\explanationSpace})$ be a metric space, and $d_W^p(P,Q)$ be the p-Wasserstein distance between two probability measures $P,Q$ on $\mathcal{E}$. Let $\refdist$ be a uniform measure on $\mathcal{E}$ (Eq. \ref{eq:uniform}), and $\eta_\mathcal{E}: \mathcal{E} \rightarrow \mathcal{E}$ be a centering mapping. Then the Wasserstein Globalness $G_p(\mu)$ is defined as follows:}
\label{def:WasG}
    \begin{equation}\label{eq:Wasserstein_def}
            G_p(\mu) \coloneqq d_W^p(\centermu, \refdist)
    \end{equation}
    \begin{equation}
        = \inf_{\pi \in \Pi(\centermu, \refdist)} \left[ \mathbb{E}_{(x,y) \sim \pi}[d_{\explanationSpace}(x,y)^p]\right]^{1/p}
    \end{equation}
\end{definition}

\setcounter{theorem}{0}
\begin{theorem} 
\label{theorem:props_satisfied}
Let $\distnset$ be the set of probability measures over $\explanationSpace$, where $\explanationSpace \subseteq \mathbb{R}^s$ (feature attribution) or $\explanationSpace \subseteq \{0,1\}^s$ (feature selection). Wasserstein Globalness $G_{p}: \distnset \to \mathbb{R}$ satisfies Properties 1-6. 
\end{theorem}
The proof of Theorem \ref{theorem:props_satisfied} is provided in App. \ref{pf:props_satisfied}. Because it satisfies these intuitive properties, \abbr{} is a satisfactory way to quantify explanation globalness.

The \abbr{} framework is highly flexible, requiring only the definition of a metric space $(\explanationSpace, d_{\explanationSpace})$.
This allows it to be applied for both feature attribution (continuous) and feature selection (discrete) cases by selecting an appropriate distance metric $d_{\explanationSpace}$.
Moreover, users have the flexibility to adopt alternative notions of explanation similarity, enabling additional transformation invariances as needed. For example, if only relative feature rankings matter and not the explicit attribution values, users can set $d_{\explanationSpace}$ to the Kendall-Tau distance and $\refdist$ to a uniform distribution over permutations. This configuration makes \abbr{} sensitive to changes in ranking but invariant to scaling.
In App. \ref{sec:recommended_spaces} Table \ref{tab:d_examples} we list possible $d_{\explanationSpace}$ and $\refdist$ selections.

Since \abbr{} is upper bounded by fully-global explainers (\ref{prop:global}), we additionally normalize $G_p$ to improve interpretability.
First, note that all fully-global explainers will have the same WG value. 
Let $\dot \mu$ be the probability measure for a global explanation; we scale $G_p(\mu)$ by $G_p(\dot \mu)$ to normalize $G_p(\dot \mu)=1$.
Since WG is zero for fully-local explainers, the WG metric is thus bounded between $[0,1]$.

\subsection{Sample Complexity} \label{sec:sample_complexity}
\vspace{-1mm}

In practice we do not have access to the true distribution $\mu$; we instead empirically estimate $\mu$ with $N$ sampled explanations.
Similarly, using an empirical approximation for the uniform distribution allows us to estimate the \abbr{} with a discretized formulation: 
\begin{definition}[Empirical Wasserstein Globalness] \emph{Let $Q^{(N)}$ be the empirical approximation for a measure $Q \in \distnset$ with $N$ samples, where $Q^{(N)} = \frac{1}{N}\sum_{i=1}^N \delta_{q_i}$ for $q_i \overset{\mathrm{iid}}{\sim} Q$. Then the empirical Wasserstein Globalness $\hat{G}_p(\mu^{(N)})$ is defined as:}
    \begin{equation} \label{eq:discrete_wasserstein}
        \hat{G}_p(\mu^{(N)}) \coloneqq d_W^p\left(\centermu^{(N)}, \refdistemp\right)
    \end{equation}
    \vspace{-8mm}
\end{definition}
We now theoretically evaluate how well $\hat{G}_p(\mu^{(N)})$ approximates $G_p(\mu)$ through sample complexity bounds:
\setcounter{theorem}{1}
\begin{theorem}\label{theorem:approximation_error} 
Let $p$ be the order of the \abbr{}, and let $M_q(\mu) = \int_{\mathbb{R}^d} |x|^q d\mu(x)$. If $p \in (0, d/2)$, $q \neq d/(d - p), \ q \geq \frac{dp}{d-p}$, then:
\begin{align*}
\mathbb{E}\left[|\hat{G}_p(\mu^{(N)}) - G_p(\mu)|\right] \leq \kappa_{p,q} M_q(\centermu)^{1/q} N^{-1/d} \\
 + C_{p,q,d} M_q^{p/q}(\centermu)(N^{-p/d} + N^{-(q-p)/q})
\end{align*}
\vspace{-7mm}
\end{theorem}
Proof details are shown in App. \ref{app:proof_sample_complexity}. From this result, we immediately see that $\mathbb{E}\big{[}|\hat G_p(\mu^{(N)}) - G_p(\mu) |\big{]} \rightarrow 0$ as $N \rightarrow \infty$.
However, increasing the number of features $d$ leads to a looser bound, necessitating an increasing number of samples as N increases.

\subsection{\abbr{} Implementation} \label{sec:implementation}
\vspace{-1mm}
In practice, a user may need to choose between different explainers for a pretrained model and dataset. While faithfulness metrics can help compare explainers, many explainers can exhibit similar faithfulness scores. In such cases, \abbr{} can be used to differentiate between high-faithfulness explainers by identifying the lower-complexity explainer.

The algorithm for \abbr{} is detailed in Appendix \ref{app:algorithm}.
In summary, we begin by selecting a distance metric $d_\explanationSpace$ and a baseline distribution $U_\explanationSpace$ based on the type of explainer and its application.
We then compute $\mu^{(N)}$ (Eq. \ref{eq:discrete_wasserstein}) by drawing a set of $N$ samples from the dataset and applying the explainer to each sample. We draw samples from $U_\mathcal{E}^{(N)}$. Finally, we compute the Wasserstein distance between $\mu^{(N)}$ and $U_\mathcal{E}^{(N)}$.

Wasserstein distances are challenging to calculate in high dimensions; we rely on approximation methods for \abbr{}.
While generally any approximation method can be used, Sliced Wasserstein \citep{bonneel2015sliced} is particularly efficient for the feature attribution case when $d_\explanationSpace$ is Euclidean distance, due to symmetry in the baseline distribution.
Alternatively, Entropic regularization \citep{sinkhorn} is another widely-used algorithm that can be used with other distance metrics. We utilize \citet{pot} for our implementation of \abbr{}.
Additional detail on approximation methods are outlined in App. \ref{app:implementation}.
Empirical time complexity results are provided in App. \ref{app:time}.
\section{EXPERIMENTS} \label{sec:experiments}
\vspace{-2mm}

\begin{figure*}[t]
  \centering
  \includegraphics[width=1\textwidth]{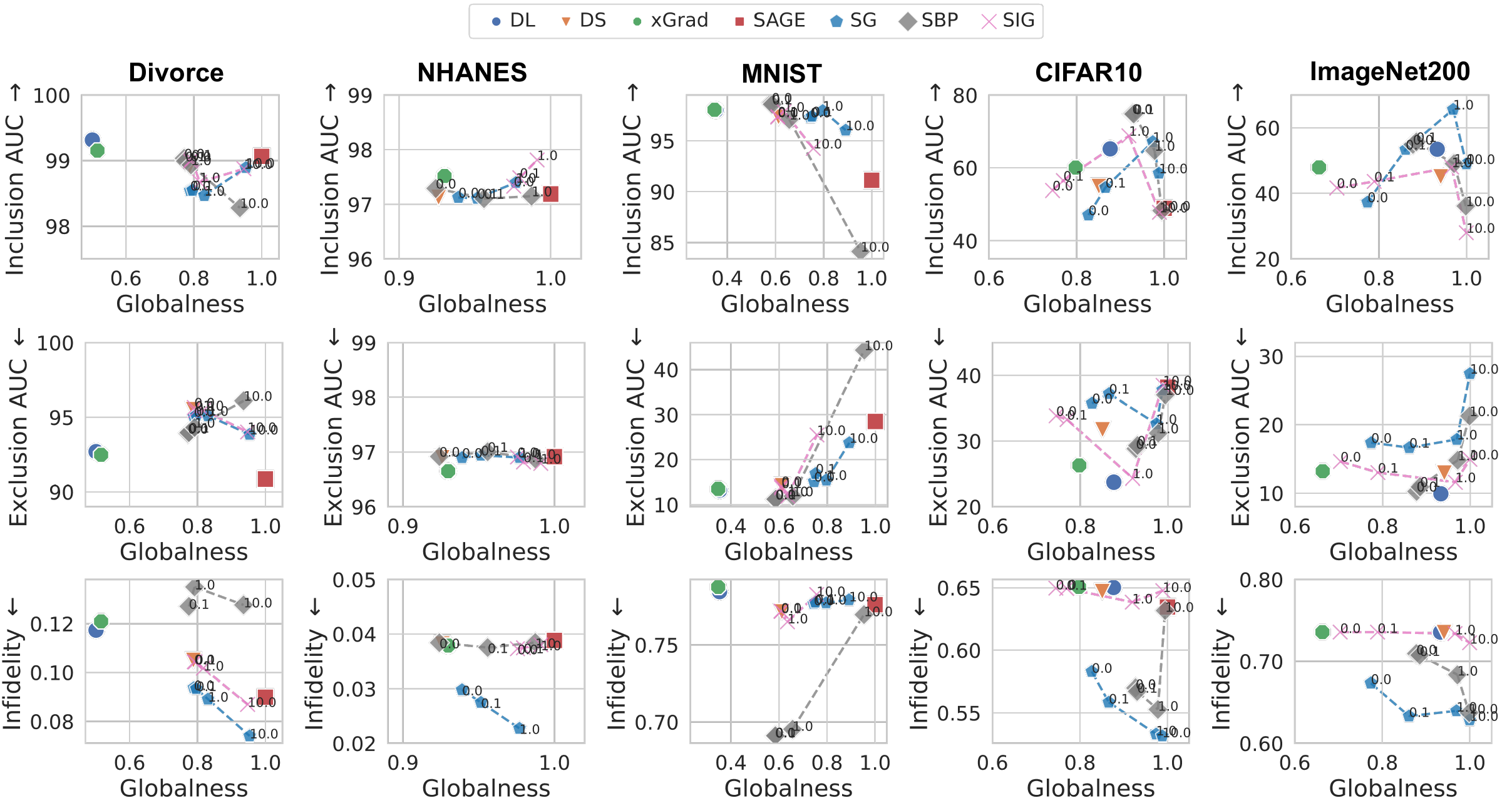}
  \vspace{-1mm}
  \caption{A comparison of explanation \emph{faithfulness} (IncAUC, ExAUC, and Infidelity) and \emph{globalness}.
  We vary the smoothing parameter $\sigma = \{0.0, 0.1, 1, 10\}$ for SG, SBP, and SIG and plot with a connecting line. 
  We observe that explainers can often exhibit similar faithfulness with varying levels of globalness, especially for simpler datasets. In this situation, we would prefer high-faithfulness explainers with higher globalness (i.e. lower complexity).
  }
  \label{fig:auc}
\end{figure*}

\begin{figure*}[t]
  \centering
  \includegraphics[width=0.98\textwidth]{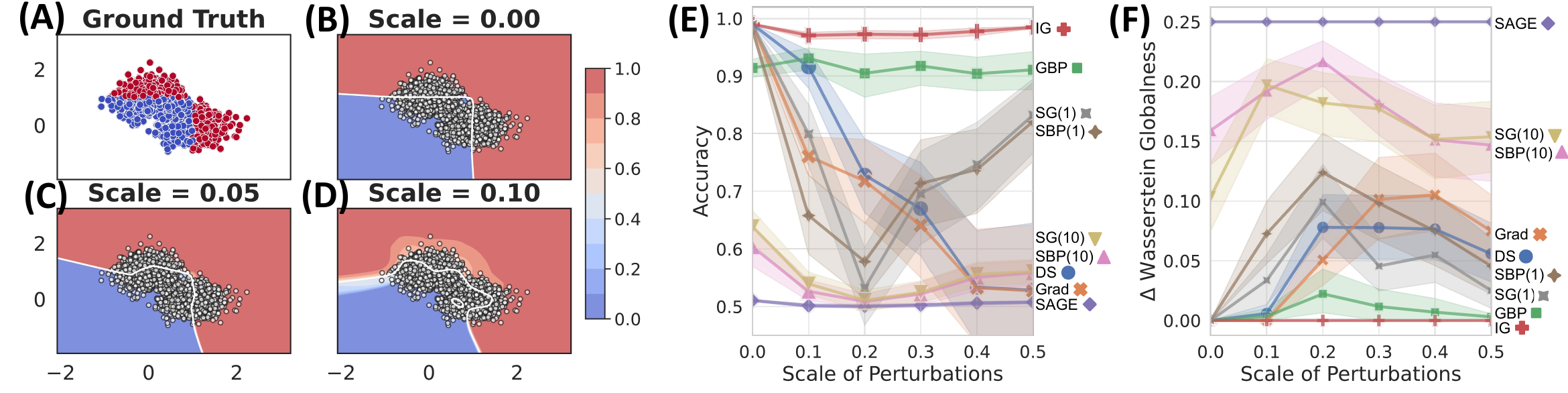}
  \vspace{-2mm}
  \caption{A classifier is trained on the Jagged Boundary Dataset \textbf{(A)} with increasing scale of perturbations. The model probability output is plotted as a heatmap \textbf{(B-D)}, with white decision boundary. \textbf{(E)} Comparison of explainer accuracy with increasing perturbation scale. Error bars indicate 95\% CI over 50 sampled perturbations. \textbf{(F)} Difference in \abbr{} score between each explainer and the ground truth explanations. IG is the closest to the ground truth globalness ($\Delta$ = 0) while also yielding the highest accuracy.}
  \label{fig:jagged_boundary}
\end{figure*}

We evaluate \abbr{} on a variety of explainers, datasets and classifiers. In Sec. \ref{sec:experiment_auc} we show how various explainers can be compared based on faithfulness metrics against their level of globalness. In Sec. \ref{sec:experiment_jagged_boundary} we illustrate how \abbr{} helps in evaluating explainer effectiveness. In Sec. \ref{sec:experiment_clustering} we show how \abbr{} captures increased explanation diversity.
We approximate \abbr{} with $p=2$, (i.e. $\hat G_2$) and select $d_\mathcal{E}, U_\mathcal{E}$ based on the relevant framework (Sec. \ref{sec:axioms}).
All experiments were run on an internal cluster using AMD EPYC 7302 16-Core processors with Nvidia V100 GPUs.
All source code is publicly available at \url{https://github.com/davinhill/WassersteinGlobalness}.

\textbf{Datasets.} We evaluate \abbr{} on 2 tabular dataset (Divorce \citep{uci_datasets}, NHANES \citep{millerPlanOperationHealth1973}), 3 image datasets (MNIST \citep{lecunMNISTHandwrittenDigit2010}, CIFAR10 \citep{krizhevskyLearningMultipleLayers2009}, ImageNet200 \citep{imagenet}), and 1 synthetic dataset (Jagged Boundary, described in Section \ref{sec:experiment_jagged_boundary}).
Datasets and model details are listed in App. \ref{app:datasets}.

\textbf{Explainers.} 
We apply \abbr{} on a variety of explainers, including Saliency map (Grad) \citep{Simonyan2014DeepIC}, Integrated Gradients (IG) \citep{sundararajanAxiomaticAttributionDeep2017}, DeepSHAP (DS) \citep{shap_lundberg}, Guided Backpropagation (GBP) \citep{guided_backprop}, and SAGE \citep{covert2020understanding}. \citet{smoothgrad} propose a smoothing operation for explanation methods: given a user defined $\sigma >0$, 
point of interest $x \in \inputspace$ and explanation method $E : \inputspace \rightarrow \explanationSpace$ the smoothed explainer may be written as:
\vspace{-2mm}
\begin{equation} \label{eq:smoothgrad}
    E_{\sigma \text{-smooth}}(x) = \mathbb{E}_{\delta \sim \mathcal{N}(0, \sigma^2 I)}[E(x + \delta)]
    \vspace{-2mm}
\end{equation}
We apply Eq. \ref{eq:smoothgrad} to Saliency map (abbreviated SG$_\sigma$), Guided Backpropagation (SBP$_\sigma$), and Integrated Gradients (SIG$_\sigma$) to generate smoothed explanations with varying $\sigma$.
Note that SG$_0$, SBP$_0$, and SIG$_0$ are equivalent to Grad, GBP, and IG, respectively.

Explainer details are listed in App. \ref{app:explainers}.

\subsection{Comparing Faithfulness and Globalness} \label{sec:experiment_auc}
\vspace{-1mm}

In Figure \ref{fig:auc} we explore the relationship between \abbr{} and three established faithfulness metrics for explainers: \emph{Inclusion} / \emph{Exclusion} AUC (incAUC / exAUC) \citep{jethani2022fastshap}, and Infidelity \citep{yeh2019fidelity}.
Exclusion AUC iteratively masks the most important features from the inputs, calculates the accuracy between the original samples and masked samples, then sums the area under the resulting curve. 
Conversely, IncAUC starts with all features masked, iteratively unmasks the most important features, then similarly evaluates accuracy with respect to the original samples. 
Higher IncAUC (resp. ExAUC) indicates better (resp. worse) explainer performance. 
Infidelity estimates how well explanations capture the change in model output under perturbation to the model inputs; higher values indicate worse performance. Standard deviation results are provided in App. \ref{app:auc_sd}.

First, note that \abbr{} increases as the smoothing parameter $\sigma$ for SG, SBP, and SmoothIG increases. 
This follows intuition; increasing $\sigma$ averages explanations over wider neighborhoods, increasing globalness.
Higher values of $\sigma$ produce explanations with globalness close to $1$, which indicates maximum globalness.

Secondly, explainers with low globalness (e.g. xGrad) do not necessarily have the best performance; this is dependent on the dataset.
For smaller datasets, such as Divorce, explainers with higher globalness (e.g. SAGE) generally achieve lower Infidelity and ExAUC. 
In contrast, explainers with lower globalness (e.g. SG$_{1}$ and SBP$_1$) achieve better performance on more complex datasets such as CIFAR10 and ImageNet. However, explainers with minimal globalness (SG$_{0}$ and SIG$_0$) also perform relatively poorly on these datasets.

Finally, we observe that many explainers achieve similar IncAUC/ExAUC/Infidelity performance with varying levels of globalness. For example, DS, GBP, IG, and SG$_{1.0}$ achieve Inclusion AUC between 98.1 to 98.2 on MNIST, however SG$_{1.0}$ produces the most global explanations of this group.  For the same level of AUC performance, one should choose the more global explanation.
 In summary, explainers are historically compared on masking and perturbation-based performance metrics, however this experiment reveals that we have been comparing apples and oranges.  Instead it is crucial that we assess explainer efficacy while considering similar levels of performance.

\subsection{Synthetic Ground Truth Globalness} \label{sec:experiment_jagged_boundary}
\vspace{-1mm}

While explainers are often utilized to understand a given black-box prediction model, they have also been used to gain insights into the underlying data distribution \citep{catav2021marginal, cui2022gene, masoomi2022explanations}. 
In this setting, the black-box model is assumed to capture the ground-truth dependency between features and labels; explanation methods are then utilized to infer the ``true'' feature importance scores from data.

In this experiment, we evaluate the importance of choosing an explainer that reflects the ``true globalness'' of the underlying distribution.
To this end, we create a 2-d synthetic dataset, Jagged Boundary (Figure \ref{fig:jagged_boundary}(A)), where samples are defined to have a ground-truth explanation. 
Jagged Boundary contains two Gaussian clusters in $\mathbb{R}^2$; each cluster is divided into two classes, and each cluster has a different relevant feature for discriminating between classes, which is defined as the ground-truth explanation. 
After generating Jagged Boundary, we increasingly perturb the dataset to simulate added noise (Figure \ref{fig:jagged_boundary}(B-D)) and retrain the classifier, obfuscating the ground truth explanations (App. \ref{app:experiment_details}). 

We evaluate the ability of each explainer $E$ to select the single most relevant feature for each sample. 
As ground-truth explanations are known, we can compute explainer accuracy as $Acc(E) = \frac{1}{N} \sum_{i=1}^N \mathds{1}[\argmax_{k \in \{1,2 \}} |E(x_i)_k| = e(x_i)]$ (Figure \ref{fig:jagged_boundary}(E)), where $e(x_i)$ represents the ground-truth explanation for sample $x_i$.
We additionally compare each explainer's WG with the WG of the ground-truth explanations; the difference ($\Delta$ Wasserstein Globalness) is shown in Figure \ref{fig:jagged_boundary}(F).

We observe that the ground truth globalness is low; the two clusters are balanced, therefore explanations should be uniformly distributed. 
The best-performing explainer, IG, exhibits minimal globalness with low variance over perturbations. 
In contrast, the global explainer SAGE is unable to capture the two explanation groups and exhibits low accuracy. 
Therefore, explainers that captures similar globalness levels as the underlying data also improve accuracy.

\subsection{WG on Explanation Groups} \label{sec:experiment_clustering}
\vspace{-1mm}

Intuitively, \abbr{} captures how diverse a set of explanations are; a single global explanation for the entire dataset has no diversity, whereas  uniformly distributed explanations would yield minimal globalness.
In this experiment, we evaluate how well \abbr{} captures increasing diversity of explanations.
To simulate increasing explanation diversity, we generate explanations for 1000 samples from MNIST, CIFAR10, and ImageNet, then separate explanations by class label.
Each subset of sample explanations indicate feature attributions with respect to each subset's class label.
We then evaluate the \abbr{} while iteratively including additional subsets of explanations.

Experiment results are shown in Figure \ref{fig:clustering_experiment}. 
As expected, \abbr{} generally decreases as the number of explanation classes increases, indicating a higher diversity of explanations. 
Notably, this effect is particularly evident for explainers that initially have lower WG score. 
For example, Grad, GBP, and DS exhibit low WG for a single class and also have high sensitivity (i.e., large decrease) as additional classes are added. 
In contrast, SAGE (\abbr{} $= 1.0$), maintains the same level of globalness regardless of the number of  explanation classes.
Explainers with high globalness (e.g. SBP$_{10}$) also exhibit minimal decrease in globalness as explanation diversity increases. 
These results indicate that \abbr{} is able to capture explanation diversity.

\begin{figure*}[t]
  \centering
  \includegraphics[width=0.94\linewidth]{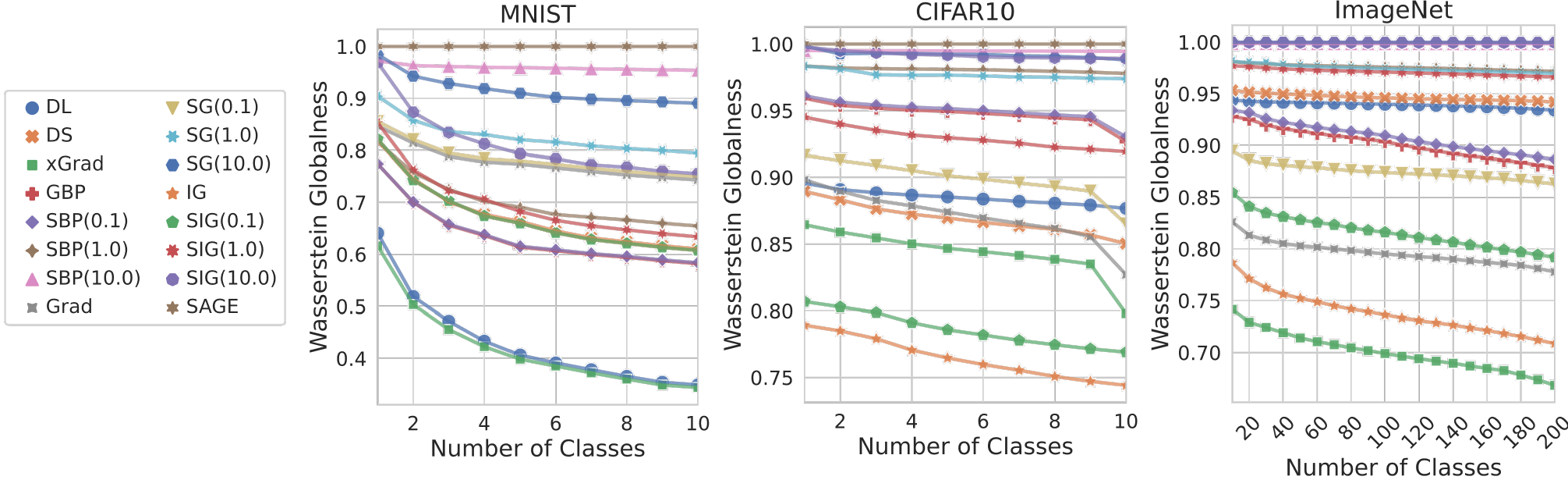}
  \vspace{-3mm}
  \caption{We evaluate whether WG can capture an increasing diversity of explanations for the MNIST, CIFAR10, and ImageNet200 datasets. As we increase the number of label classes in a set of samples, the corresponding explanations become more diverse, resulting in a decreasing WG score. Explainers with higher WG score (e.g. SAGE, SBP(10.0)) are less sensitive to the increasing diversity of explanations from adding additional classes, which is observed from the lower rate of decrease in WG.}
  \label{fig:clustering_experiment}
\end{figure*}

\subsection{Additional Results} \label{sec:additional_experiments}
\vspace{-1mm}
In the supplement, we include a time complexity evaluation of \abbr{} (App. \ref{app:time}), ablation study (App. \ref{app:ablation}), extended results on ImageNet 1k (App. \ref{app:imagenet1k}), examples of isometries for feature attribution and selection (App. \ref{app:additional_isometry_figure}), and qualitative examples from image datasets (App. \ref{app:qualitative_images}).
\section{CONCLUSION} \label{sec:conclusion}
\vspace{-2mm}
In this work, we introduce \emph{globalness}, an explainer complexity metric designed to differentiate between high-faithfulness explainers by evaluating the diversity of their explanations.
Globalness deepens our understanding of explainers by quantifying their expressiveness with respect to a given dataset and black-box model.
We further introduce \emph{Wasserstein Globalness}, an explainer-agnostic approach for measuring globalness, and demonstrate that it satisfies six desirable axioms.
Regarding limitations, Wasserstein distance approximation can be challenging in high dimensions and is an ongoing area of research.
We use modern approximation methods in our implementation (App. \ref{app:implementation}) and leave efficiency improvements for future work.

\subsubsection*{Acknowledgements}
This work was supported in part by grants NIH 2T32HL007427-41, U01HL089856, and 5R01HL167072 from the National Heart, Lung, and Blood Institute, and grants NIH R01CA240771 and U24CA264369 from the National Cancer Institute. We thank the AISTATS reviewers for their useful comments and feedback.


\bibliographystyle{unsrtnat}
\bibliography{dh_references}

\clearpage
\section*{Checklist}

 \begin{enumerate}

 \item For all models and algorithms presented, check if you include:
 \begin{enumerate}
   \item A clear description of the mathematical setting, assumptions, algorithm, and/or model. [Yes]
   \item An analysis of the properties and complexity (time, space, sample size) of any algorithm. [Yes]
   \item (Optional) Anonymized source code, with specification of all dependencies, including external libraries. [Yes]
 \end{enumerate}

 \item For any theoretical claim, check if you include:
 \begin{enumerate}
   \item Statements of the full set of assumptions of all theoretical results. [Yes]
   \item Complete proofs of all theoretical results. [Yes]
   \item Clear explanations of any assumptions. [Yes]   
 \end{enumerate}

 \item For all figures and tables that present empirical results, check if you include:
 \begin{enumerate}
   \item The code, data, and instructions needed to reproduce the main experimental results (either in the supplemental material or as a URL). [Yes]
   \item All the training details (e.g., data splits, hyperparameters, how they were chosen). [Yes]
         \item A clear definition of the specific measure or statistics and error bars (e.g., with respect to the random seed after running experiments multiple times). [Yes]
         \item A description of the computing infrastructure used. (e.g., type of GPUs, internal cluster, or cloud provider). [Yes]
 \end{enumerate}

 \item If you are using existing assets (e.g., code, data, models) or curating/releasing new assets, check if you include:
 \begin{enumerate}
   \item Citations of the creator If your work uses existing assets. [Yes]
   \item The license information of the assets, if applicable. [Not Applicable]
   \item New assets either in the supplemental material or as a URL, if applicable. [Not Applicable]
   \item Information about consent from data providers/curators. [Not Applicable]
   \item Discussion of sensible content if applicable, e.g., personally identifiable information or offensive content. [Not Applicable]
 \end{enumerate}

 \item If you used crowdsourcing or conducted research with human subjects, check if you include:
 \begin{enumerate}
   \item The full text of instructions given to participants and screenshots. [Not Applicable]
   \item Descriptions of potential participant risks, with links to Institutional Review Board (IRB) approvals if applicable. [Not Applicable]
   \item The estimated hourly wage paid to participants and the total amount spent on participant compensation. [Not Applicable]
 \end{enumerate}

 \end{enumerate}

\clearpage
\appendix
\onecolumn

\section{Background} \label{app:background}

\subsection{Broader Impact Statement} \label{app:broader_impact}
Machine learning models are increasingly used in variety of applications, such as healthcare \citep{rasheed2022explainable,deep_learning_utilizing_suboptimal_spirometry, torop2020deep}, financial services \citep{bucker2022transparency,bracke2019machine}, criminal justice \citep{dressel2018accuracy}, or game design \citep{summerville2018procedural, bazzaz2024guidedgamelevelrepair, bazzaz2024controllablegamelevelgeneration}. In these applications, it is often important for users to understand how predictions generated in order to ensure transparency, trust, and informed decision-making.
Explainability is one way in which we can begin to build trustworthy models. It is crucial that we seek a better understanding of opaque models to ensure their safe deployment in real-world scenarios, whether by constructing inherently interpretable models or by applying post-hoc explanation methods. In this paper, we discuss a theoretical property of explainability methods, and present a technique for measuring this quantity in practice. We intend for our proposed globalness measure to facilitate responsible practices by improving prediction models and their interpretability methods.

\subsection{Additional Related Works}
In this section, we survey previous works that reference different aspects of what we term ``globalness''. While many methods exhibit varying degrees of globalness, there have been no previous attempts at quantification.
To the best of our knowledge, our work is the first to define and formalize a continuous measure of explanation globalness.

Various instance-wise explainers (e.g. \citet{chen2018learning, ribeiroLIME, shap_lundberg}) have been investigated with respect to the diversity of their explanations.
\citet{anchors} propose Anchors, which evaluates the ``coverage'' of  individual explanations and how well they generalize to nearby samples.
Alternatively, some methods explain subsets of samples by averaging instance-wise explanations \citep{invase, masoomi2022explanations} or through clustering \citep{lundbergLocalExplanationsGlobal2020a, why_buy_insurance_gramegna}.
The explainer LIME \citep{ribeiroLIME} has a bandwidth parameter which has been related to the local neighborhood of each explanation \citep{sLIME, visani2020optilime}.
Other methods use similar smoothing operations that average nearby explanations \citep{smoothgrad, torop2023smoothhess}.
\citet{plumb2018model} uses a surrogate tree model to identify interpretable global patterns as well as individual explanations.

\subsection{Summary of Notation}

Table \ref{tab:notation} provides a summary of the notation and terminology used in the paper.

\begin{table}[h!]
    \centering
    \begin{tabular}{| c | p{0.6\linewidth} |} 
     \hline
     \textbf{Symbol} & \textbf{Name} \\ [0.5ex] 
     \hline
      $G_p$ & p-Wasserstein Globalness (Eq. \ref{eq:Wasserstein_def}).  \\ 
     \hline
      $\hat G_p$ & Empirical p-Wasserstein Globalness (Eq. \ref{eq:discrete_wasserstein}).  \\ 
     \hline
     $\phi_{\#}$ & Push-forward Measure.  \\ 
     \hline
     $\mathds{1}$ & Indicator Function. \\
     \hline
     $\delta_x$ & Dirac distribution (centered at x). \\ 
     \hline
     $\mathbb{E}$ & Expected Value. \\
     \hline
     $\explanationSpace$ & Set of Explanations. \\
     \hline
     $\inputspace$ & Space of model inputs. \\
     \hline
     $\explainer$ & Explainer. \\
     \hline
     $\distnset$ & Space of probability measures over $\mathcal{E}$.\\
     \hline
     $\Pi(P,Q)$ &  Set of joint distributions which marginalize to $P$ and $Q$. \\
     \hline
     $\langle \cdot, \cdot \rangle_F$ & Frobenius Inner Product. \\ [1ex] 
     \hline
     $d_{W}^p$ & p-Wasserstein distance. \\
     \hline
     $d_{\explanationSpace}$ & distance metric for the space $\mathcal{E}$. \\
     \hline
     $\refdist$ & Baseline uniform distribution for the space $\mathcal{E}$.\\
     \hline
     $u_\mathcal{E}$ & Probability density/mass function for $\refdist$.\\
     \hline
     $\eta_\mathcal{E}$ & Borel measurable mapping for $\mu$; represents a translation to center $\mu$ for $\mathcal{E}=\mathbb{R}^s$ (feature attribution). In the feature selection case, it is set to be an identity mapping. \\
     \hline
     supp$(\zeta)$ & The support of a given probability measure $\zeta$.\\
     \hline
     $\mathfrak{S}_s$ & The symmetric group on $\{1,\ldots,s\}$. \\
     \hline
    \end{tabular}
    \caption{Summary of notation used in the paper}
    \label{tab:notation}
\end{table}

\subsection{Approximating Wasserstein Globalness in High-Dimensions} \label{app:implementation}

Computing the Wasserstein distance traditionally involves solving a linear programming problem which does not scale well. In the previous section, we discussed the sample complexity of $\hat{G}_p(\mu_N)$. This empirical estimate also scales poorly with the dimension of the data. However, several techniques have been developed to efficiently compute Wasserstein distances in high dimensions. One such technique, called entropy-regularized optimal transport, relies on regularizing the entropy of the joint distribution $\pi \in \Pi(P,Q)$. 

The Wasserstein distance can be written as an optimization over the set of product measures $\Pi(P,Q)$. In its discrete form, the expectation $\mathbb{E}_{(x,y) \sim \pi}[d_{\explanationSpace}(x,y)]$ from Equation \ref{eq:Wasserstein_def} can be written as a Frobenius inner product between the matrix of distances $D$, and the transport plan $R$.

\begin{equation} \label{eq:orig_opt}
    d_W^1(p,q) = \min_{R \in \Pi(p,q)} \langle R,D \rangle_F
\end{equation}

The \emph{sinkhorn distance}, introduced to the ML community by \citet{sinkhorn}, utilizes a technique called entropic regularization. The entropy-regularized transport problem adds an additional regularization term $\Omega(R) = \sum_{i,j} r_{i,j}log(r_{i,j})$. 

\begin{equation} \label{eq:entropy_regularization}
    d_{\textrm{sinkhorn}}(p,q) = \min_{R \in \Pi(p,q)} \langle R,D \rangle_F + \lambda \Omega(R)
\end{equation}

This encourages the transport plan to have less mutual information with the marginals $p$ and $q$. By searching just over these smooth transport plans, the original optimization problem becomes convex and enables the use matrix scaling algorithms \footnote{We implement the sinkhorn algorithm using available optimal transport solvers available at \citet{pot}.}. There are also existing theoretical guarantees that $d_{sinkhorn}$ will be close to $d_W$, for which we refer the reader to \citet{sinkhorn}. 
By utilizing entropic regularization, we can find good solutions faster.  However, we can decrease the strength of this regularization to reduce noise from the computation. As discussed, this is especially useful for applications on high-dimensional data.

Alternatively, Sliced Wasserstein Distance (SWD) \citep{bonneel2015sliced} approximates Wasserstein Distance by projecting each measure $p,q$ onto random 1-dimensional unit vectors (i.e. slices), calculating the 1-d p-Wasserstein Distance on each slice, then taking an expectation over all slices. One-dimensional p-Wasserstein distance is relatively efficient to calculate and can be rewritten in terms of Cumulative Distribution Functions (CDF) of $p,q$, leading to an simplified expression for SWD. 
\begin{equation} \label{eq:sliced_wasserstein_distance}
    d_{\textrm{SWD}}(p,q) = \left(\int_0^1 d\left(\inv{F_p(x)}, \inv{F_q(x)}\right)^p d x\right)^{\frac{1}{p}}
\end{equation}
where $F_p$ and $F_q$ are the CDFs for $p$ and $q$, respectively.
The integral in \eqref{eq:sliced_wasserstein_distance} is commonly approximated using Monte Carlo methods, by averaging a finite sample of random vectors. In our experiments we use the implementation of SWD provided by \citep{pot}.

While both Sinkhorn Distance and SWD can be used for calculating Wasserstein Globalness for both feature attribution and feature selection, the SWD approximation is particularly efficient for the feature attribution case. This is because the baseline uniform distribution for attribution $U_A$ is defined as a uniform distribution over a $s$-dimensional ball of radius $k$. 
Due to the symmetry of $U_A$, the projection onto each slice is symmetric for all slices, enabling efficient approximation of the global measure used for normalization. 
Enforcing symmetric projections also improves invariance to isometric transformations in practice.

\section{Theory and Method Details} \label{app:proofs}
\subsection{Proof of Theorem \ref{theorem:props_satisfied}} \label{pf:props_satisfied}
We will now restate and prove that the axioms from Section \ref{sec:axioms} apply to the definition of $G$ in Definition \ref{def:WasG}. For clarity, we omit the $\mathcal{E}$ in $U_{\mathcal{E}}$, however this still refers to either uniform measure from the feature attribution or feature selection case.

\textbf{\ref{prop:non-negativity}} states that $G_p(\mu) \geq 0 \ \forall \mu$.

\begin{proof}
Since $d_W$ is a distance metric, it is non-negative by construction. Therefore $G_p(P) = d_W(P,U)$ is non-negative for any $P,U \in \distnset$. \qed

\textbf{\ref{prop:continuity}} Let $\{\mu^{(n)}\}$ be a sequence of probability measures which converges weakly to $\mu$. Then $G_p(\mu^{(n)}) \rightarrow G_p(\mu)$. 

\textit{Proof.} 

For clarity, we write the push-forward measure $\centermu$ as $\centermuclean$, and the push-forward measure $\centermu^{(n)}$ as $\centermuclean^{(n)}$.

By remark 2.4 in \citep{memoli},
\begin{equation}
    |d_W^p(A,B) - d_W^p(A^{(n)}, B^{(n)})| \leq d_W^p(A, A^{(n)}) + d_W^p(B, B^{(n)})
\end{equation}
This implies that 
\begin{equation}
    |d_W^p(\centermuclean, U) - d_W^p(\centermuclean^{(n)}, U^{(n)})| \leq d_W^p(\centermuclean, \centermuclean^{(n)}) + d_W^p(U, U^{(n)})
\end{equation}

However, $U$ does not need to be approximated, therefore 
\begin{align*}
    U = U^{(n)} \implies |G_p(\mu) - G_p(\mu^{(n)})| &= |d_W^p(\centermuclean, U) - d_W^p(\centermuclean^{(n)}, U)| \\
    &\leq d_W^p(\centermuclean, \centermuclean^{(n)})
\end{align*}

But $\mu^{(n)}$ converges weakly to $\mu$, meaning $d_W^p(\centermuclean, \centermuclean^{(n)}) \rightarrow 0 \implies |G_p(\mu) - G_p(\mu^{(n)})| \rightarrow 0$
\end{proof}

From this proof we can see that, not only does $G_p(\mu^{(n)})$ converge to $G_p(\mu)$, but it converges at least as quickly as $\mu^{(n)} \rightarrow \mu$. 

\textbf{\ref{prop:global}} states that there exists an $x_0 \in \explanationSpace$ for which $G_p(\delta_{x_0}) \geq G_p(\mu) \ \forall \mu \in P(\explanationSpace)$.
  Furthermore, we will also show that, for this $x_0$, $s_{x_0,p} \leq s_{x,p} \ \forall x \in \explanationSpace$, where $s_{x,p}$ is the \textit{p-eccentricity}, defined in \citet{memoli} as $s_{x,p} = \int_X d^p(x, x')\; dx'$ \footnote{P-eccentricity indicates how far a point $x$ is to the other points in $\explanationSpace$.}.
  
P-eccentricity appears in this property in order to characterize the $x_0$ for which $\delta_{x_0}$ achieves maximum globalness.

\begin{proof}. 
Without loss of generality, assume $s_{x_0, p} \geq s_{x, p}, \ \forall x \in \explanationSpace$.
Any $\textbf{discrete}$ distribution can be represented as a delta-train $P = \sum_i \lambda_i \delta(x_i)$ such that $\lambda_i \geq 0, \ \forall i$ and $\sum_i \lambda_i = 1$. 
By \ref{prop:convexity}, we know that $G_p(P) \leq \sum_i \lambda_i G_p(\delta(x_i))$. But for a Dirac measure, \ref{eq:Wasserstein_def} reduces to 
\begin{equation} \label{eq:delta_globalness}
    G_p(\delta_{x_0}) = \left( \int_{X} d^p(x_0, x) \; dx \right)^{1/p} = s_{x_0, p}
\end{equation}

So 
\begin{equation} G_p(P) \leq \sum_i \lambda_i s_{x_0, p} \leq \max_i s_{x_i, p} = s_{x_0, p}.
\end{equation}

This proves the axiom for any discrete distribution. But general distributions can be approximated with these weighted delta-trains, so by \ref{prop:continuity}, so in the limit, this proof holds for general distributions.
\end{proof}

\textbf{\ref{prop:local}} states that $\centermu = \refdist \iff G(\mu) = 0$
\begin{proof} First note that $d_W^p$ is a metric. Therefore it follows that $d_W^p(P, Q) = 0 \iff P = Q$ for any distributions $P,Q$. This directly implies that $d_W(\centermu, 
\refdist) = 0 \iff \centermu = \refdist$. 
\end{proof}

We will now prove the \textbf{first part of \ref{prop:isometry_invariance}}, restated below:

Let $T_{\explanationSpace, d_{\explanationSpace}}$ be the group of isometries of $(\explanationSpace, d_{\explanationSpace})$ defined in \ref{prop:isometry_invariance}. Then $G_p$ is $T_{\explanationSpace, d}$-invariant. 

\begin{proof}
We need to show that $G_p(\mu) = G_p(\phi_\# \mu), \ \forall \phi \in T_{\explanationSpace, d}$.   

First, consider the set of translation isometries $R_{\mathcal{E},d} \subset T_{\explanationSpace, d}$.
Note that due to the centering function $\eta_\mathcal{E}$, $\phi _\#(\centermu) = \centermu$ $\forall \phi \in R_{\mathcal{E},d}$. Therefore:
\begin{equation}  \label{eq:prop6pt1}
    G_p(\mu) = d_W^p\left(\phi _\#(\centermu), \refdist\right)= d_W^p\left(\centermu, \refdist\right) = G_p(\phi_\# \mu), \ \forall \phi \in R_{\mathcal{E},d}
\end{equation}

Next, we consider the set of isometries $T_{\explanationSpace, d} \setminus R_{\mathcal{E},d}$.
Let $P_q$ be a probability measure with density $q$. One consequence of the definition of push-forward measure is that \begin{equation} \label{eq:pushforward_expectation}
P_{\phi_{\#} q}(\phi x) = P_q (\phi^{-1} \phi x) = P_q (x) \implies \mathbb{E}_{x \sim \phi_\# q}[\inv{\phi} x] = \mathbb{E}_{x \sim q}[x]
\end{equation}

Consider the set of measure couplings (feasible transport plans) between $\mu$ and $v$, denoted $\Pi(\mu, v)$. Let $h_\phi : \Pi(\mu, v) \rightarrow \Pi(\phi_\# \mu, \phi_\# v)$ be a bijection such that $(h_\phi \circ \pi)(A,B) = \pi(\phi^{-1} A , \phi^{-1} B)$.
\begin{equation}G_p(\mu) = \left[\underset{\pi \in \Pi(\centermu, U)}{\inf} \mathbb{E}_{(x,y) \sim \pi}[d_{\mathcal{E}}(x, y)^p]\right]^{1/p}
\end{equation}

From \eqref{eq:pushforward_expectation}, $\mathbb{E}_{x \sim \pi}[x] = \mathbb{E}_{x \sim h_\phi(\pi)}[\phi(x)]$, it follows:

\begin{equation}= \left[\underset{\pi \in \Pi(\centermu, U)}{\inf} \mathbb{E}_{(x,y) \sim h_\phi(\pi)}[d^p(\inv{\phi} x, \inv{\phi} y)]\right]^{1/p}
\end{equation}

Since $\phi$ is an isometry, $d_{\explanationSpace}(x,y) = d_{\explanationSpace}(\inv{\phi} x, \inv{\phi}y)$:

\begin{equation}= \left[\underset{\pi \in \Pi(\centermu, U)}{\inf} \mathbb{E}_{(x,y) \sim h_\phi(\pi)}[d_{\mathcal{E}}(x, y)^p]\right]^{1/p}\end{equation}

Finally, since $h_\phi$ is a bijection between $\Pi(\centermu, U)$ and $\Pi(\phi_\# \mu, \phi_\#U)$,
\begin{equation}
    = \left[\underset{\pi \in \Pi(\phi_\# (\centermu), \phi_\#U)}{\inf} \mathbb{E}_{(x,y) \sim \pi}[d_{\mathcal{E}}(x, y)^p]\right]^{1/p}
\end{equation}

Note that $\phi_\#U = U$ for any non-translation isometry $\phi \in T_{\explanationSpace, d} \setminus R_{\mathcal{E},d}$ due to symmetry.

\begin{equation} \label{eq:prop6pt2}
    = G_p(\phi_\# \mu)  \quad \forall \phi \in T_{\explanationSpace, d} \setminus R_{\mathcal{E},d}
\end{equation}

Combining \eqref{eq:prop6pt1} and \eqref{eq:prop6pt2}, we obtain the desired result.

\end{proof}

We will now prove the \textbf{second part of \ref{prop:isometry_invariance}}, which states that $\exists \phi : \explanationSpace \rightarrow \explanationSpace \not \in T_{\explanationSpace, d_{\explanationSpace}}$ for which $G_p(\phi_\# \mu) \neq G_p(\mu)$, for some $\mu$.  
\begin{proof}

Consider a $\phi \in S$, $\phi : \explanationSpace \rightarrow \explanationSpace$ such that $d_{\explanationSpace}(x,y) = c * d_{\explanationSpace}(\phi x, \phi y)$ for some constant $c$, and some $x,y \in \explanationSpace$. Note that $\phi \notin T_{\explanationSpace, d_{\explanationSpace}}$, the group of isometries of $(\explanationSpace, d_{\explanationSpace})$.

Like before, we can write:

\begin{equation}
G_p(\mu) = \left[\underset{\pi \in \Pi(\centermu, U)}{\inf} \mathbb{E}_{(x,y) \sim \pi}[d_{\mathcal{E}}(x, y)^p]\right]^{1/p}
\end{equation}

\begin{equation}
= \left[\underset{\pi \in \Pi(\phi_\# (\centermu), \phi_\# U)}{\inf} \mathbb{E}_{(x,y) \sim \pi}[d_{\mathcal{E}}(x, y)^p]\right]^{1/p}
\end{equation}

But since $\phi$ is not isometric: 

\begin{equation}
= \left[\underset{\pi \in \Pi(\phi_\# (\centermu), \phi_\# U)}{\inf} \mathbb{E}_{(x,y) \sim \pi}[c^p d_{\mathcal{E}}(x, y)^p]\right]^{1/p}
\end{equation}

\begin{equation}
\neq G_p(\phi_\# \mu)
\end{equation}

\end{proof}

\textbf{\ref{prop:convexity}} (convexity) Let $R$ be a convex combination of measures $P,Q$ for some $\lambda \in [0,1]$. Then $G(R) \leq \lambda G(P) + (1 - \lambda) G(P)$.

\begin{proof} Let $P,Q \in \distnset$. Let $\psi_P$ be the optimal transport plan between $P$ and $U$. Similarly, let $\psi_Q$ be the optimal transport plan between $Q$ and $U$. Then the convex combination of $P$ and $Q$ is a feasible transport plan between $R = \lambda P + (1 - \lambda) Q$ and $U$.

\begin{equation}
G_p(R) = \left[\underset{\pi \in \Pi(\centermu, U)}{\inf} \mathbb{E}_{(x,y) \sim \pi}[d_{\mathcal{E}}(x,y)^p]\right]^{1/p}
\end{equation}

\begin{equation}
\leq \left[\mathbb{E}_{(x,y) \sim \lambda \psi_P + (1 - \lambda) \psi_Q}[d_{\mathcal{E}}(x,y)^p]\right]^{1/p}
\end{equation}

\begin{equation}
= \lambda \left[\mathbb{E}_{(x,y) \sim \psi_P}[d_{\mathcal{E}}(x,y)^p]\right]^{1/p}  + (1 - \lambda) \left[\mathbb{E}_{(x,y) \sim \psi_Q}[d_{\mathcal{E}}(x,y)^p]\right]^{1/p}
\end{equation}

\begin{equation}
= \lambda G_p(P) + (1 - \lambda) G_p(Q)
\end{equation}

\end{proof}

\subsection{Proof: Sample Complexity (Theorem \ref{theorem:approximation_error})} \label{app:proof_sample_complexity}

We will first introduce two lemmas to simplify the proof of theorem \ref{theorem:approximation_error}. Lemma \ref{lemma:mu_approx} concerns the error incurred by approximating $\mu$ with finite samples. Lemma \ref{lemma:empirical_convergence} concerns the error incurred by approximating $\refdist$ with finite samples. By bounding both of these errors, we obtain an upper bound on the error of using both finite-sample approximations.

For clarity, we write the push-forward measure $\centermu$ as $\centermuclean$, and the push-forward measure $\centermu^{(N)}$ as $\centermuclean^{(N)}$.

\begin{lemma}\label{lemma:mu_approx}
Let $p \in (0, d/2)$. Also, let $q \neq d/(d - p)$.
Let $M_q(\mu) = \int_{\mathbb{R}^d} |x|^q d\mu(x)$. Then for some constant $C_{p,q,d}$ which depends only on p,q, and d,

\begin{equation}
\mathbb{E}\left[|G_p(\mu) - G_p(\mu^{(N)})|\right] \leq  C_{p,q,d} M_q^{p/q}(\mu)(N^{-p/d} + N^{-(q-p)/q})
\end{equation}
\end{lemma} 

\begin{proof}

From the triangle inequality,
\begin{equation}
    d_W^p(\centermuclean^{(N)}, U) \leq d_W^p(\centermuclean, U) + d_W^p(\centermuclean^{(N)}, \centermuclean)
\end{equation}

From \citep{fournier}, $\mathbb{E}[d_W^p(\centermuclean^{(N)}, \centermuclean)] \leq C_{p,q,d} M_q^{p/q}(\mu)(N^{-p/d} + N^{-(q-p)/q})$.

Therefore: 
\begin{equation}\label{mu_approx_lem1}
\mathbb{E}[G_p(\mu^{(N)})] - G_p(\mu) \leq  C_{p,q,d} M_q^{p/q}(\mu)(N^{-p/d} + N^{-(q-p)/q})
\end{equation}

Similarly, we can invoke the triangle inequality with $d_W^p(\mu, \refdist)$ on the LHS:
\begin{equation}
d_W^p(\centermuclean, \refdist) \leq d_W^p(\centermuclean^{(N)}, \refdist) + d_W^p(\centermuclean^{(N)}, \centermuclean)
\end{equation}

And, making the same argument, we get:

\begin{equation}\label{mu_approx_lem2}
G_p(\mu) - \mathbb{E}[G_p(\mu^{(N)})] \leq  C_{p,q,d} M_q^{p/q}(\mu)(N^{-p/d} + N^{-(q-p)/q})
\end{equation}

Finally, combining \ref{mu_approx_lem1} and \ref{mu_approx_lem2},
$$\mathbb{E}\left[|G_p(\mu) - G_p(\mu^{(N)})|\right] \leq  C_{p,q,d} M_q^{p/q}(\mu)(N^{-p/d} + N^{-(q-p)/q})$$
\end{proof}

This theorem is related to \ref{prop:continuity} in the sense that the empirical distribution $\mu^{(N)}$ converges to $\mu$. This theorem is concerned with the \textit{rate} at which $G_p(\mu^{(N)})$ converges, whereas \ref{prop:continuity} simply states that it will converge.

Perhaps the most important consequence of \ref{lemma:mu_approx} is that $\mathbb{E}\left[|G_p(\mu) - G_p(\mu^{(N)})|\right] \rightarrow 0$ as $N \rightarrow \infty$. 

\begin{lemma}\label{lemma:empirical_convergence}
Let $p \in (0, \frac{d}{2})$, and $q > \frac{dp}{d-p}$. Then for some constant $\kappa_{p,q}$ which depends only on p and q,

\begin{equation}
    \mathbb{E}\left[|\hat{G}_p(\mu^{(N)}) - G_p(\mu^{(N)})|\right] \leq \kappa_{p,q} \left[ \int_{\mathbb{R}^d} \|x\|^q d\mu(x) \right] ^{1/q} N^{-1/d}
\end{equation}
\end{lemma}

\begin{proof}
From the triangle inequality:
\begin{align}
&d_W(\centermuclean^{(N)}, U^{(N)}) \leq d_W(\centermuclean^{(N)}, U) + d_W(U^{(N)}, U) \\
\implies &d_W(\centermuclean^{(N)}, U^{(N)}) - d_W(\centermuclean^{(N)}, U) \leq d_W(U^{(N)}, U)
\end{align}

From \citep{dereich}, there exists some constant $\kappa_{p,q}$ such that $\mathbb{E}[d_W^p(U^{(N)}, U)] \leq \kappa_{p,q} \left[ \int_{\mathbb{R}^d} \|x\|^q d\mu(x)\right] ^{1/q} N^{-1/d}$. 

$\implies \mathbb{E}[d_W(\centermuclean^{(N)}, U^{(N)}) - d_W(\centermuclean^{(N)}, U)] \leq \mathbb{E}[d_W^p(U^{(N)}, U)] \leq \kappa_{p,q} \left[ \int_{\mathbb{R}^d} \|x\|^q d\centermuclean(x)\right] ^{1/q} N^{-1/d}$  

$\implies \mathbb{E}[\hat{G}_p(\centermuclean^{(N)}) - G_p(\centermuclean^{(N)})] \leq \kappa_{p,q} \left[ \int_{\mathbb{R}^d} \|x\|^q d\mu(x)\right] ^{1/q} N^{-1/d}$

Similarly:
\begin{align*}
    &d_W(\centermuclean^{(N)}, U) \leq d_W(\centermuclean^{(N)}, U^{(N)}) + d_W(U^{(N)}, U) \\
\implies &d_W(\centermuclean^{(N)}, U) - d_W(\centermuclean^{(N)}, U^{(N)}) \leq d_W(U^{(N)}, U)
\end{align*}

Therefore:

\begin{equation}
    \mathbb{E}[|\hat{G}_p(\mu^{(N)}) - G_p(\mu^{(N)})|] \leq \kappa_{p,q} \left[ \int_{\mathbb{R}^d} \|x\|^q d\centermuclean(x) \right] ^{1/q} N^{-1/d}
\end{equation}

\end{proof}

We will now combine the previous lemmas to get a final bound on the approximation error incurred by using discrete approximations for both $\mu$ and $U$.

\begin{proof}
The following one-line proof is a natural consequence of Lemma \ref{lemma:mu_approx} and Lemma \ref{lemma:empirical_convergence}.

\begin{align*}
\mathbb{E}[|&\hat{G}_p(\mu^{(N)}) - G_p(\mu)|] \leq \mathbb{E}[|\hat{G}_p(\mu^{(N)}) - G_p(\mu^{(N)})|] + \mathbb{E}\left[|G_p(\mu) - G_p(\mu^{(N)})| \right] \\
&\leq \kappa_{p,q} \left[\int_{\mathbb{R}^d} \|x\|^q d\centermuclean(x) \right] ^{1/q} N^{-1/d} + C_{p,q,d} M_q^{p/q}(\centermuclean)(N^{-p/d} + N^{-(q-p)/q}) \\
&= \mathbb{E}[|\hat{G}_p(\mu^{(N)}) - G_p(\mu)|] \leq \kappa_{p,q} M_q(\centermuclean)^{1/q} N^{-1/d} + C_{p,q,d} M_q^{p/q}(\centermuclean)(N^{-p/d} + N^{-(q-p)/q})
\end{align*}
\end{proof}

\subsection{Proof: Entropy and F-Divergence under Feature Permutation} \label{sec:pf_fdiv}
\begin{proof}
Let $\psi \in S$ be a permutation of the set of explanations.

By Equation \ref{eq:pushforward_expectation}, a direct consequence of the definition of push-forward measure, we have that:

\begin{equation}
    \mathbb{E}_{x \sim \mu}[g(\mu(x))] = \mathbb{E}_{x \sim \psi_{\#} \mu}[g(\psi_{\#}\mu(x))]
\end{equation}

for some function $g$.

Now observe that both entropy and f-divergences, as defined in Section \ref{sec:candidates}, can be written in the form above. Entropy is obtained by directly letting $g(x) = -\log(x)$.

\begin{equation}
    H(X) = \mathbb{E}_{x \sim \nu}[-\log(\nu(x))]
\end{equation}

The f-divergence between $\mu$ and $U$ is equivalent by allowing $g(x) = f(\frac{c}{x})$, and allowing the constant $c$ to be the uniform probability of an element $x$, $\nu(x)$.

\begin{equation}
    D(\nu \| u) = \mathbb{E}_{x \sim u}\left[f(\frac{\nu(x)}{u(x)})\right]
\end{equation}
\end{proof}

\subsection{Recommended Distance Metrics and Reference Distributions} \label{sec:recommended_spaces}

\abbr{} allows the user to specify a distance metric $d_\explanationSpace$ and baseline distribution $U_\explanationSpace$ based on the explanation space and intended application.
Selecting $d_\explanationSpace$ allows the user to specify a notion of similarity between explanations, affecting how distributional invariances are captured. The baseline distribution determines the minimally-global explanation distribution.
In the main text (Sec. \ref{sec:axioms}), we assume $\explanationSpace = \mathbb{R}^s$ for attribution and $\explanationSpace = \{0,1\}^s$ for selection, as these are the most widely applicable scenarios. However, in Table \ref{tab:d_examples}, we provide alternative choices for $d_\explanationSpace$ and $U_\explanationSpace$ tailored to common explanation frameworks.

\textbf{Defining bounds on $\bf U_\explanationSpace$.} Note that when $\explanationSpace$ is unbounded (e.g. $\explanationSpace = \mathbb{R}^s$, $\explanationSpace = \mathbb{R}^s_{\geq 0}$) and $U_\explanationSpace$ is selected to be the uniform distribution over $\explanationSpace$, it is necessary to bound the support of $U_\explanationSpace$.
Specifically, given a set of explainers $\{E_1,\ldots,E_m\}$ to be compared, with corresponding explanation distributions $M = \{\mu_{(E_1,F,X)}, \ldots, \mu_{(E_m,F,X)}\}$, it is important to define the bound of $\textrm{supp}(U_\explanationSpace)$ such that $\textrm{supp}(\mu) \subseteq \textrm{supp}(U_\explanationSpace) \quad \forall \mu \in M$.
As described in Section \ref{sec:axioms}, we define a constant $k>0$ determined by $\{E_1,\ldots,E_m\}$: $k = \max_{\zeta \in \textrm{M}} [ \sup_{x \in \textrm{supp}(\zeta)} ||x||_2]$. We detail how $k$ is calculated in practice in Algorithm \ref{alg:k}.

\textbf{Comparing ranked/relative feature attributions.} In some cases, a user may be interested in the relative ranking of feature attribution values rather than the values themselves.
In this case, \abbr{} enables the direct comparison of feature rankings by selecting an appropriate $d_\explanationSpace$ and $U_\explanationSpace$.
Let $\mathfrak{S}_s$ denote the symmetric group on $\{1,\ldots, s\}$, which describes all possible permutations of the feature rankings.
First, we convert each explanation $E(x) \in \mathbb{R}^s$ to a ranking $E(x) \in \mathfrak{S}_s$ (i.e. argsort).
We then define $d_\explanationSpace$ to be a distance metric between rankings, such as Kendall-Tau distance. This definition allows comparison between explanations based \emph{only} on feature attribution ranking.
We then define $U_\explanationSpace$ to be the uniform distribution over $\mathfrak{S}_s$.
Intuitively, this setup enables \abbr{} to quantify the diversity of feature rankings generated by $E$ across the dataset $X$.

\begin{table}[h!]
    \centering
    \small
    \renewcommand{\arraystretch}{1.5}
    \setlength{\tabcolsep}{6pt}
    \begin{tabular}{l l p{4cm} l  L{3.6cm}} 
    \toprule
     \bf{Framework} & $\bf \explanationSpace$  & \bf{Distance Measure} $\bf d_\explanationSpace$ & \bf{Baseline Distribution} $\bf U_\explanationSpace$ & \bf{Description}\\ 
     \midrule
     
      Attribution & $\reals^s$ & $d_{A}(x,y) = ||x - y||_2$ & $u_A(x;k) \propto  \begin{cases}1 & ||x||_2 \leq k \\ 0 & \textrm{otherwise} \end{cases}$ & Default option defined in Section \ref{sec:axioms}. Works for any set of feature attribution explainers.\\

      Attribution & $\mathfrak{S}_s$ & $d_{A}(x,y)$=   $|\{\textrm{Discordant Pairs in } x \times y\}|$ & $u_A(x) \propto  \begin{cases}1 & x \in \explanationSpace \\ 0 & \textrm{otherwise} \end{cases}$ & For applications where only relative feature rank, and not attribution values, is important.\\ 
     
     
      Selection & $\{0,1\}^s$ & $d_{S}(x,y) = \sum_i \mathds{1}(x_i \neq y_i)$ & $u_S(x) \propto  \begin{cases}1 & x \in \explanationSpace \\ 0 & \textrm{otherwise} \end{cases}$ & Default option defined in Section \ref{sec:axioms} for feature selection explainers. \\ 
     \bottomrule
    \end{tabular}
    \caption{Suggested $d_\explanationSpace$ and $U_\explanationSpace$ options for common explanation frameworks. The term $\mathfrak{S}_s$ denotes the symmetric group of on $\{1,\ldots,s\}$. The term $k > 0$ is a constant to bound the support of $U_\explanationSpace$ for certain $\explanationSpace$; this is calculated per Alg. \ref{alg:k}.}
    \label{tab:d_examples}
\end{table}

\begin{algorithm}[h!]
\small
\setlength{\textfloatsep}{10pt}
   \caption{Estimate $k$ for $u_A(x;k)$}
   \label{alg:k}
    \textbf{Input :} Dataset $X$, Explainers $E_1,\ldots,E_m$. \\ 
    \textbf{Output :} Baseline Distribution $U_\explanationSpace$. \\

    $X^{(N)} \leftarrow x_1, \ldots, x_N \sim X \quad \quad \backslash \backslash$ Draw $N$ samples from the dataset. \\

    \For{$i \in 1,\ldots, m$}{
    $\mu^{(N)}_i \leftarrow E_i(X^{(N)}) \quad \quad \quad \quad  \; \quad \quad \backslash \backslash$ Generate explanations from each explainer. 
    
     $\alpha_i \leftarrow \max || \mu^{(N)}_i ||$ $\quad \quad \quad \quad  \: \quad \quad \backslash \backslash$ Calculate the maximum $\ell_2$ norm for each set of explanations. \\ 
    }
   
    $k = \max \{ \alpha_i \}_{i=1}^{m}$ \\
    
    Return $k$
    
\end{algorithm}

\subsection{Wasserstein Globalness Algorithm} \label{app:algorithm}

The algorithm for \abbr{} is outlined in Alg. \ref{alg:wasg}.
\abbr{} is calculated for a given dataset, black-box model, and explainer.
The user must specify a distance metric $d_\explanationSpace$ and baseline distribution $U_\explanationSpace$ (see App. \ref{sec:recommended_spaces}).
The \abbr{} algorithm also assumes access to a Wasserstein distance solver, which we denote as function $d_W: \mathcal{P}(\mathcal{\explanationSpace}) \times \mathcal{P}(\mathcal{\explanationSpace}) \rightarrow \mathbb{R}$. Wasserstein distance approximation methods are discussed in App. \ref{sec:implementation}. Any Wasserstein distance solver can be used with the \abbr{} algorithm provided that the choice of $d_\explanationSpace$ is supported.

\begin{algorithm}[h!]
\small
\setlength{\textfloatsep}{10pt}
   \caption{Wasserstein Globalness Approximation}
   \label{alg:wasg}
    \textbf{Input :} Dataset $X$, Explainer $E$, Distance Metric $d_\explanationSpace$, Baseline Distribution $U_\explanationSpace$, Centering Function $\eta_\mathcal{E}$ (Optional), Wasserstein Distance Approximator $d_W$. \\ 
    \textbf{Output :} Wasserstein Globalness $\hat G_2$. \\

    $X^{(N)} \leftarrow x_1, \ldots, x_N \sim X \quad \quad \backslash \backslash$ Draw $N$ samples from the dataset.
    
    $\mu^{(N)} \leftarrow E(X^{(N)}) \quad \quad \quad \quad  \; \quad \quad \backslash \backslash$ Generate explanations. 
    
    $U_\explanationSpace^{(N)} \leftarrow u_1, \ldots, u_N \sim U_\explanationSpace \quad \quad \backslash \backslash$ Sample from baseline distribution.    \\

    \If{$\eta_\explanationSpace$ is not NONE}{$\mu^{(N)} \leftarrow \eta_\explanationSpace(\mu^{(N)})$ $\quad \quad \; \; \quad  \quad \backslash \backslash$ Center explanations, if specified.}

    $\hat G_2 \leftarrow d_W(\mu^{(N)}, U_\explanationSpace^{(N)}; d_\explanationSpace)$ $\quad \quad \; \; \; \backslash \backslash$ Calculate Wasserstein distance w.r.t. specified metric $d_\explanationSpace$.
    
    Return $\hat G_2$
\end{algorithm}

\section{Experiment Details} \label{app:experiment_details}

\subsection{Datasets and Models} \label{app:datasets}
We use 2 tabular datasets, 1 synthetic dataset, and 3 image datasets. All non-synthetic datasets are normalized to have a mean of zero and standard deviation of 1 prior to training.

\textbf{Divorce.} The Divorce dataset \citep{uci_datasets}, consists of a 54-question survey from 170 participants. Each survey asks participants to rank (scale 1-5) responses to questions related to activities and attitudes toward their partner. We use the dataset to predict divorce. We train a multi-layer perceptron (MLP) binary classifier with 2 hidden layers of width 50. The model achieves test accuracy of 98.5\%.

\textbf{NHANES.} The NHANES dataset \citep{millerPlanOperationHealth1973} is an annual survey conducted by the National Center for Health Statistics (NHCS), containing survey data for 3,329 patients. We train a Tabnet Neural Network model \citep{arik2021tabnet} to predict type-2 diabetes using a pre-selected set of 27 features \citep{dinhDatadrivenApproachPredicting2019}.

\textbf{MNIST.} MNIST \citep{lecunMNISTHandwrittenDigit2010} consists of 70000 grayscale images of dimension 28x28. There are 10 classes; each class corresponds to a handwritten digit against a black background. We train a neural network classifier with 2 convolution layers (200 channels, kernel size 6, stride 2) and a single fully-connected layer. The model achieves 98.3\% train accuracy and 98.5\% test accuracy.

\textbf{CIFAR10.} CIFAR10 \citep{krizhevskyLearningMultipleLayers2009} consists of 60,000 color images of dimension 32x32. Each of the 10 classes corresponds to a different object (airplane, automobile, bird, cat, deer, dog, frog, horse, ship, truck). We train a convolutional neural network using the Resnet-18 architecture \citep{resnet}. The model achieves 99.9\% train accuracy and 81.8\% test accuracy.

\textbf{ImageNet200.} The ImageNet200 dataset \citep{imagenet} consists of color images with 200 classes that we scale to 224x224. We use a standard Resnet-50 architecture with the last layer modified to predict 200 classes. We use the OpenOOD benchmark \citep{openood} in training the model.

\textbf{Jagged Boundary Dataset.} The Jagged Boundary dataset is a synthetic dataset consisting of two Gaussian clusters. The two clusters are split in half; each half corresponds to a different class. We then perturb the data: algorithm \ref{alg:jagged_boundary} describes the jagged-boundary data modification shown in Figure \ref{fig:jagged_boundary}. We retrain a multi-layer perceptron (MLP) binary classifier for each perturbation. The MLP contains 5 hidden layers of width 300.

\begin{algorithm}[h!]
\caption{Jagged-Boundary Synthetic Dataset} 
\label{alg:jagged_boundary}
    Let $X$ be a finite data sample.   
    
    Let $label(\cdot)$ give the label of each point in $X$.  
    
    Let $N_{x}$ be a neighborhood around point $x$.  
    
    \While{not all points visited}{
        \For {unvisited $x$}{
            \For {unvisited $x' \in N_{x}$}{
                $label(x') \leftarrow label(x)$
                $x'$ visited
            }
        }
    }
\end{algorithm}

\subsection{Explainers} \label{app:explainers}

\textbf{Saliency} \citep{Simonyan2014DeepIC} directly uses the gradient of class score with respect to the input image as a saliency map. The absolute value of the gradient is used as a feature importance score. We use the implementation provided in \citet{captum_ai}.

\textbf{Input x Gradients} \citep{shrikumar2016not} multiplies the the input sample with the gradient of the model with respect to the inputs.

\textbf{Integrated Gradients} \citep{sundararajanAxiomaticAttributionDeep2017} accumulates this gradient at several points along a path from some baseline input to the input to-be-explained. We set the baseline to be the zero vector for both image and tabular datasets, which corresponds to the mean of the training distribution. We use the implementation provided in \citet{captum_ai}.

\textbf{DeepLIFT} \citep{pmlr-v70-shrikumar17a} is a gradient-based method that computes importance scores for each feature by comparing the activation of each neuron to a reference activation value. We define the reference sample to be a vector of zeros. We use the implementation provided in \citet{captum_ai}.

\textbf{DeepSHAP} \citet{shap_lundberg} approximates Shapley values by extending \textit{DeepLIFT}. DeepSHAP requires specifying a reference sample or distribution; to reduce computational expense we calculate the mean of the training distribution and use this vector as the reference sample. We use the implementation provided in \citet{captum_ai}.

\textbf{SAGE} \citep{covert2020understanding} is a Shapley-based method that calculates a global feature attribution value. We use the implementation provided at \url{https://github.com/iancovert/sage}.

\textbf{Guided Backpropagation} \citep{guided_backprop} modifies the gradient approach by adding a signal that prevents the backwards flow of negative gradients during backpropagation. We use the implementation provided in \citet{captum_ai}.

\textbf{SmoothGrad, SmoothGBP, SmoothIG} \citep{smoothgrad} smooths the gradients with a Gaussian kernel to alleviate the high frequency fluctuations commonly seen in partial derivatives. Given a user defined $\sigma >0$, 
point of interest $x \in \inputspace$ and explanation method $\phi : \inputspace \rightarrow \explanationSpace$ the smoothed explainer may be written as
\begin{equation} \label{eq:smoothgrad_appendix}
    \phi_{\sigma \text{-smooth}}(x) = \mathbb{E}_{\delta \sim \mathcal{N}(0, \sigma^2 I)}[\phi(x + \delta)]
\end{equation}
We apply Eq. \eqref{eq:smoothgrad_appendix} to three different explainers. SmoothGrad applies smoothing to the saliency map. We also apply smoothing to Guided Backpropagation (SmoothGBP) and Integrated Gradients (SmoothIG). We use the implementation provided in \citet{captum_ai} to apply the explanation smoothing. In our experiments, we use 500 Gaussian samples to approximate the expected value in Eq. \eqref{eq:smoothgrad_appendix}.

\subsection{Figure \ref{fig:properties_example} Details} \label{app:properties_example}
We create a synthetic distribution of 1d explanations to illustrate the effects of isometric and non-isometric transformations on Wasserstein Globalness and competing metrics.

\textbf{Dataset.} We use a mixture of gaussians to simulate a multi-model distribution of explanations. We define the distribution $\nu \sim 0.5\mathcal{N}(\mu = 3, \sigma = 0.5) + 0.5 \mathcal{N}(\mu = -12, \sigma = 1.9)$

\textbf{Competing Metrics.} In practice, a user would only have access to sampled explanations from $\nu$. However, we directly calculate competing methods on the ground-truth distribution $\nu$ in order to avoid confounding effects of different discretization procedures. We estimate the the following comparison metrics using Monte Carlo integration.

We apply two f-divergences: Kullback-Leibler (KL) Divergence and Total Variation Distance. In both cases, we calculate the divergence of $\nu$ with the uniform distribution $u \sim \mathcal{U}[-30,30]$. First, we restate the f-divergence formulation between $\nu$ and $u$:
\begin{equation} 
D_f(\nu\|u) = \mathbb{E}_{X \sim u}\left[f\left(\frac{\nu(X)}{u(X)}\right)\right]
\end{equation}

KL-divergence uses $f(t) = t \log t$, resulting in the equation:
\begin{equation}
   D_{KL}(\nu,u) = \mathbb{E}_{X \sim u} \left[\log \frac{u(X)}{\nu(X)} \right]
\end{equation}

Total Variation distance uses $f(t) = \frac12 |t-1|$, resulting in the equation:
\begin{equation}
    D_{TV}(\nu,u) = \frac12 \int |u(X) - v(X)| \; \textrm{d}X
\end{equation}

We also evaluate Entropy on the explanation distribution $\nu$:
\begin{equation} 
H(X) = \mathbb{E}_{X \sim \nu}\left[-\log\left(\nu(X)\right)\right]
\end{equation}

\textbf{Wasserstein Globalness.} We directly apply Wasserstein Globalness to samples from $\nu$. We use the same reference distribution $u \sim \mathcal{U}[-30,30]$ for equal comparison, and use the same implementation settings as described in Section \ref{sec:experiments}.

\section{Additional Results} \label{app:results}
\subsection{Time Complexity Evaluation} \label{app:time}
\begin{figure}[t]
  \centering
  \includegraphics[width=0.8\linewidth]{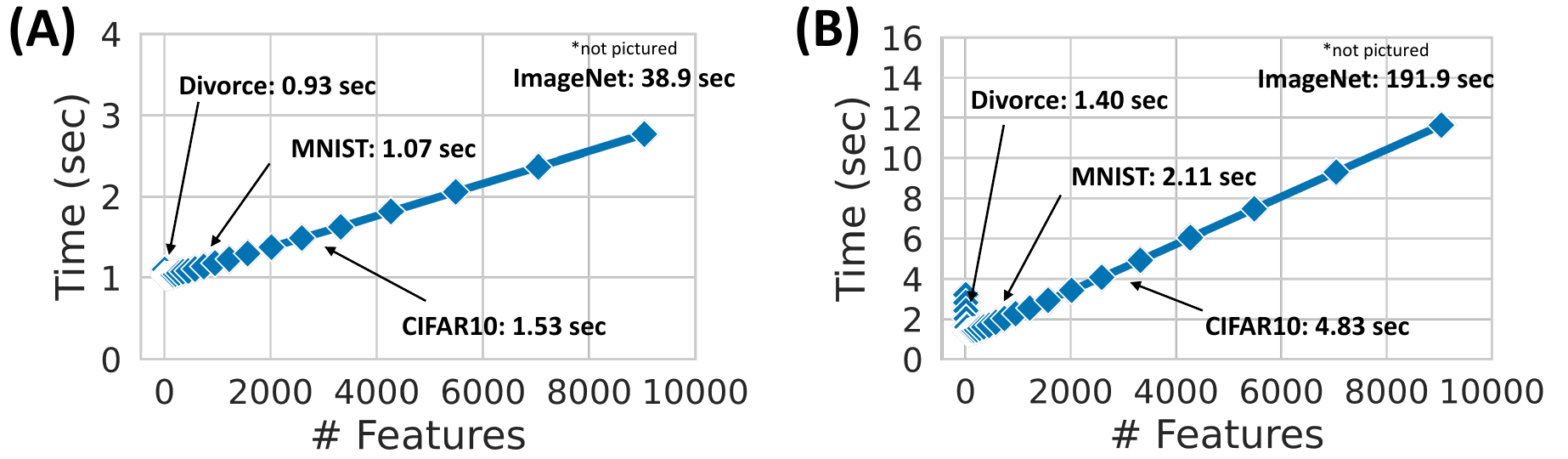}  
  \caption{Wall clock time in seconds (blue line) for Wasserstein Globalness calculated for $10^3$ explanation samples and $10^4$ uniform samples. The results in \textbf{(A)} use Sliced Wasserstein approximated using Monte Carlo sampling. The results in \textbf{(B)} use entropic regularization. We observe that clock time is independent of black-box model architecture.}
  \label{fig:time}
\end{figure}

In Figure \ref{fig:time} we show the time complexity measurements for calculating Wasserstein Globalness using two different optimal transport solvers. Both algorithms are implemented using the Python Optimal Transport package \citep{pot}, which is publicly available at \url{https://pythonot.github.io}.

\subsection{Ablation Study} \label{app:ablation}
We conduct an ablation study to evaluate \abbr{} when using different distance metrics and baseline distributions. In Figure \ref{fig:wg_ablation_synthetic} we define a synthetic set of explanations with varying levels of globalness ($\lambda$), and test whether \abbr{} can recover the relative globalness.

We first evaluate the effect of increasing $k$ in the recommended baseline distribution $u_A(x;k)\propto \mathds{1}_{||x||_2 \leq k}(x)$, (left-most column in the Figure). We evaluate \abbr{} for $k=0.66$, which is the recommended $k$, as described in Section \ref{sec:properties}. As expected, higher values of $\lambda$ correspond with lower \abbr{}. As we increase $k$, this increases the “spread” of the baseline distribution, which represents the minimally-global explanation (\ref{prop:local}). As expected, we observe that the absolute WG scores increase on average, but the relative rankings between the different explainers remains consistent.

In Figure \ref{fig:wg_ablation_divorce}, we replicated the above experiment on the Divorce dataset, using the SmoothGrad \citep{smoothgrad} explainer. Here, we vary the explainer’s globalness by decreasing the smoothing parameter $\sigma$, which generally corresponds to less global explanations (as noted in Section \ref{sec:experiment_auc}). The recommended $k$ (Section \ref{sec:properties}) is 0.83; we also evaluate $k=1$ and $k=2$. The results were consistent with the synthetic dataset, demonstrating that WG effectively quantifies globalness under different configurations.

\begin{figure*}[t]
  \centering
  \includegraphics[width=0.8\textwidth]{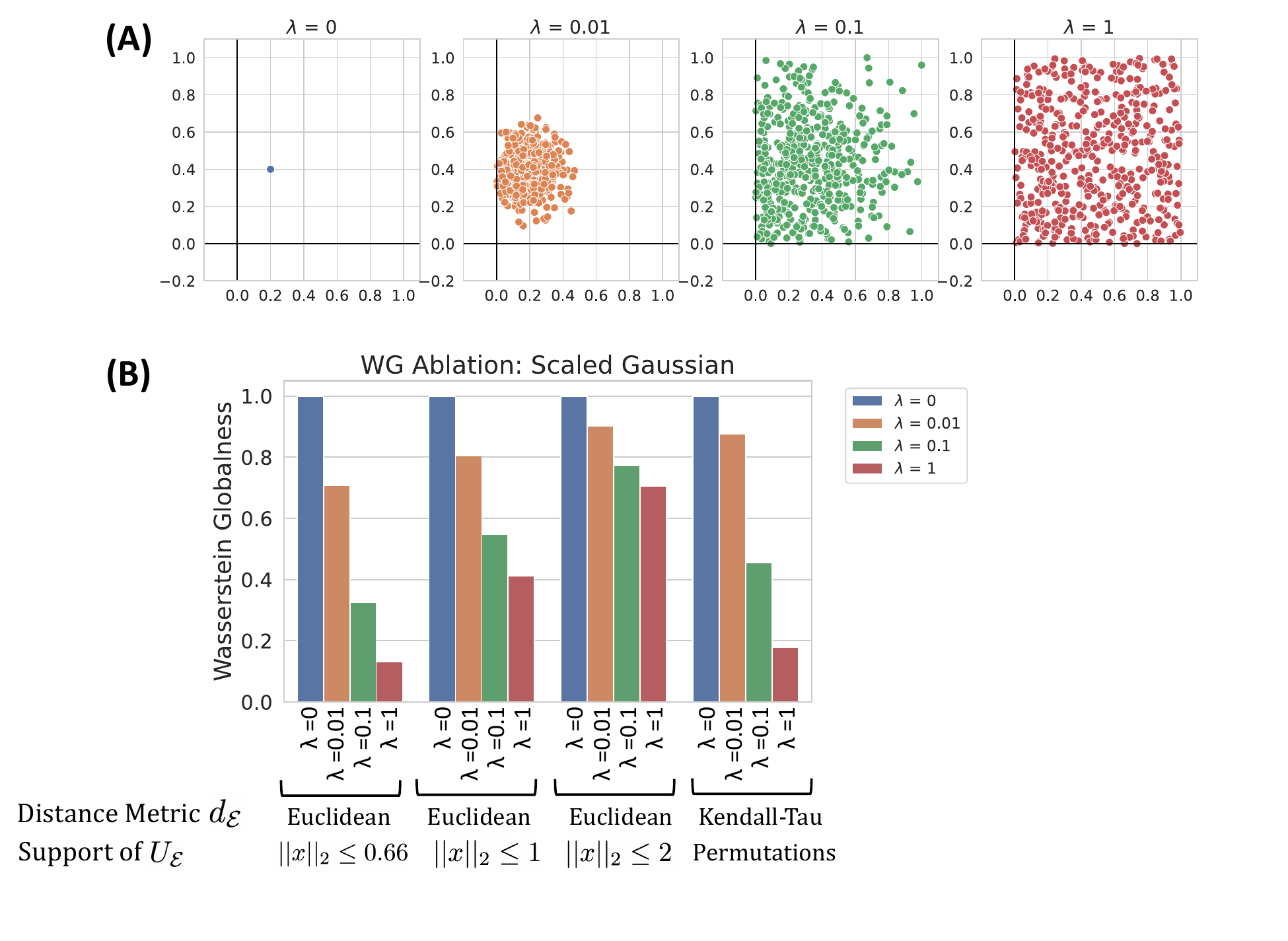}
  \vspace{-5mm}
  \caption{Ablation test on different distance metrics $d_\mathcal{E}$ and baseline distributions $U_\mathcal{E}$ for a synthetic dataset. \textbf{(A)} We construct a synthetic distribution of explanations $E(x) \sim \mathcal{N}(\begin{bmatrix}0.2  \\ 0.4  \end{bmatrix},\begin{bmatrix}\lambda & 0  \\ 0 & \lambda  \end{bmatrix})$, truncated to be within $[0,1]^2$, with decreasing globalness (increasing $\lambda$, respectively). \textbf{(B)} We calculate empirical WG using different distance metrics and baseline distributions. In the first 3 columns, we use the Euclidean distance and uniform distribution $u_A(x;k)\propto \mathds{1}_{||x||_2 \leq k}(x)$ with varying $k$. In the last column, we convert explanations to a ranking, then use the Kendall-Tau distance and the uniform distribution over permutations.}
  \label{fig:wg_ablation_synthetic}
  \vspace{-2mm}
\end{figure*}

\begin{figure*}[t]
  \centering
  \includegraphics[width=0.68\textwidth]{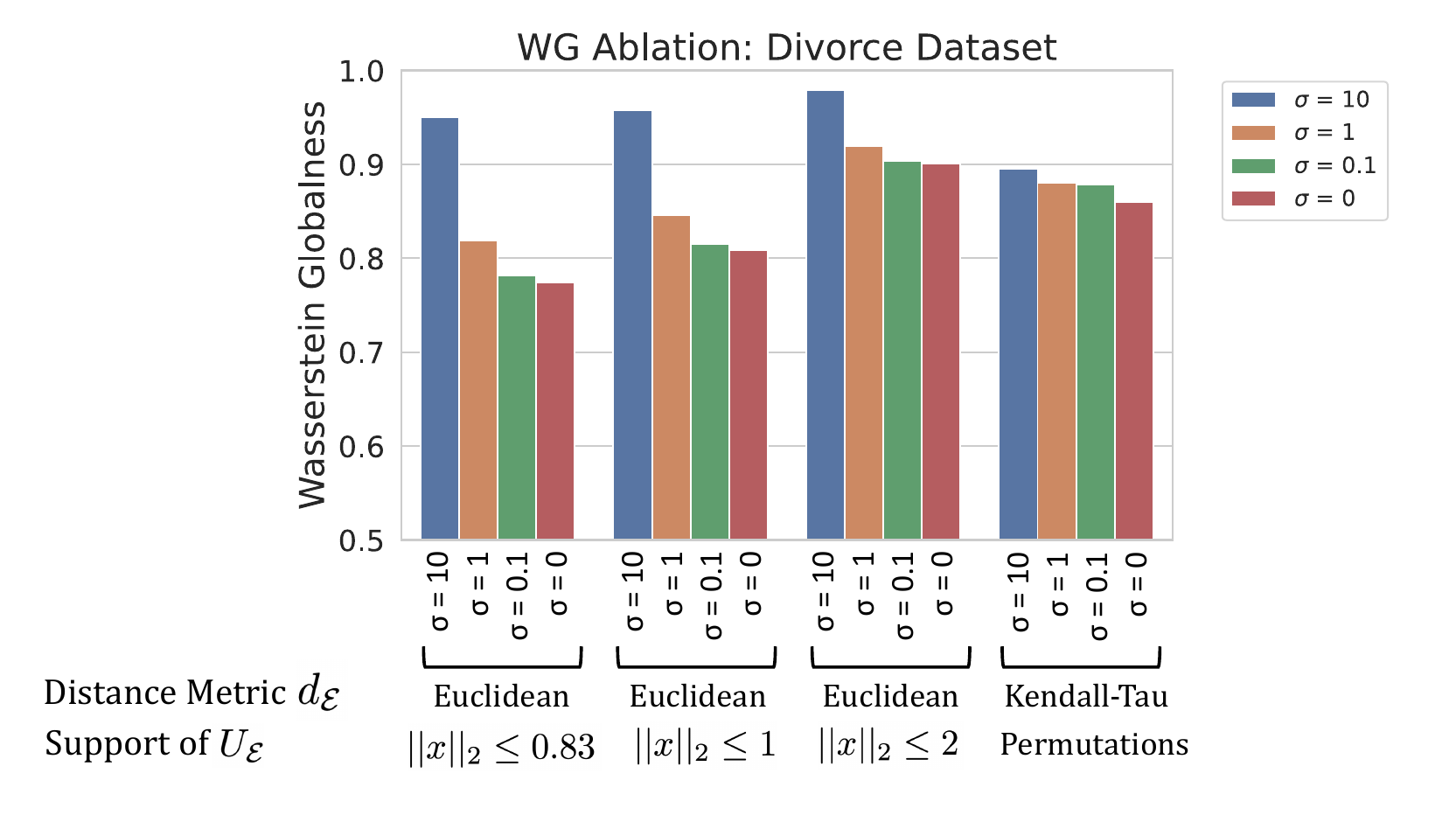}
  \vspace{-2mm}
  \caption{Ablation test on different distance metrics $d_\mathcal{E}$ and baseline distributions $U_\mathcal{E}$ for the Divorce dataset, using the SmoothGrad explainer \citep{smoothgrad}. We decrease the smoothing parameter $\sigma$ for Smoothgrad (Eq. \ref{eq:smoothgrad}), which generally corresponds to decreasing globalness (see Section \ref{sec:experiment_auc}). We calculate empirical WG using different distance metrics and baseline distributions. In the first 3 columns, we use the Euclidean distance and uniform distribution $u_A(x;k)\propto \mathds{1}_{||x||_2 \leq k}(x)$ with varying $k$. In the last column, we convert explanations to a ranking, then use the Kendall-Tau distance and the uniform distribution over permutations.}
  \label{fig:wg_ablation_divorce}
  \vspace{-2mm}
\end{figure*}

\subsection{Results on Imagenet1000} \label{app:imagenet1k}
In Figure \ref{fig:wg_imagenet1k} we extend the results of Section \ref{sec:experiment_auc} and \ref{sec:experiment_clustering} to the ImageNet1000 dataset \citep{imagenet}. The dataset description matches that of ImageNet200 (App. \ref{app:datasets}), but is extended to the full 1000 classes.
Due to the computational constraints, we limit our analysis on the more efficient attribution methods, SG, IG, and GBP.
Note that the main computational bottleneck of calculating \abbr{} stems from its need to sample the distribution of explanations corresponding to the given dataset, therefore it is dependent on the efficiency of the particular explainer.
\begin{figure*}[t]
  \centering
  \includegraphics[width=0.86\textwidth]{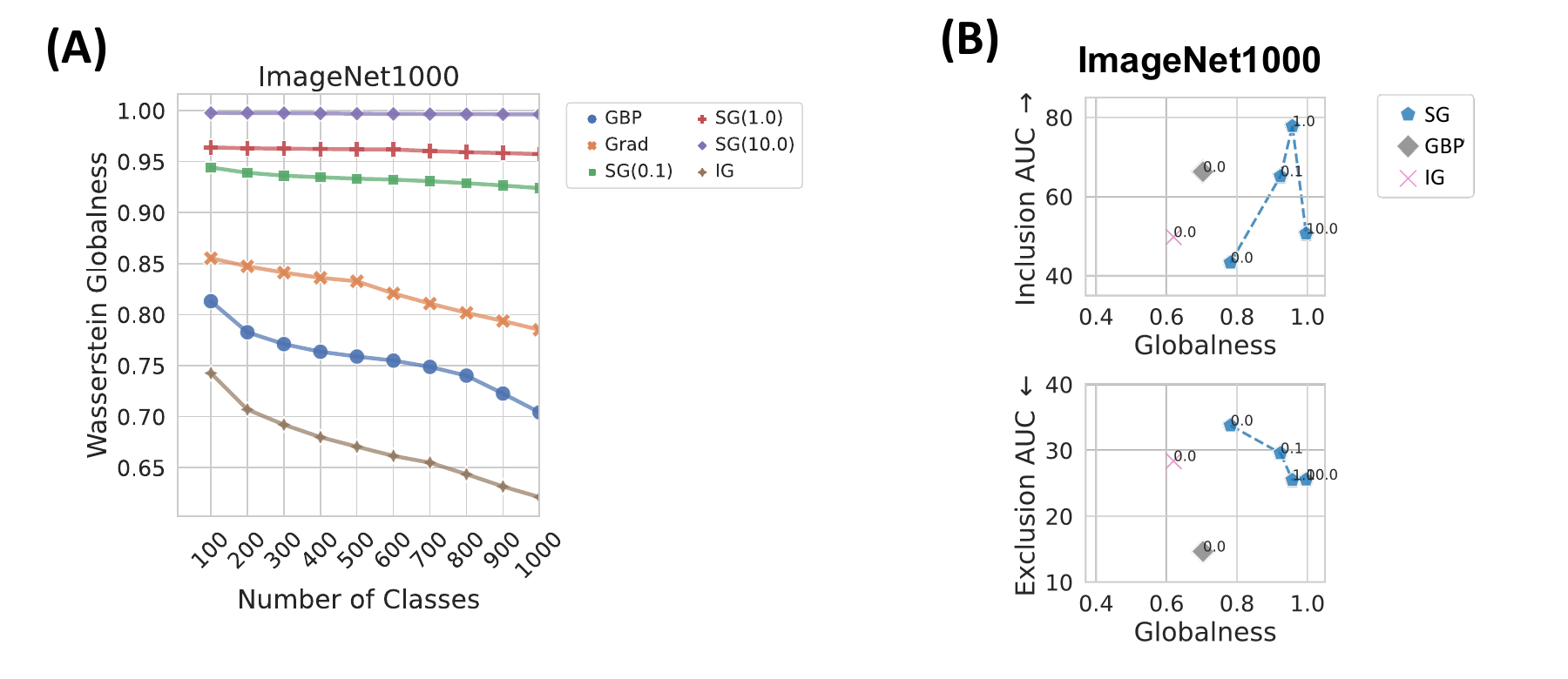}
  \vspace{-4mm}
  \caption{\textbf{(A)} Extension of the experiment in Section \ref{sec:experiment_clustering} to the ImageNet1000 dataset. We evaluate whether WG can capture an increasing diversity of explanations by taking samples from different label classes. As expected, the WG score decreases as more classes are included, indicating decreased globalness. \textbf{(B)} Extension of the experiment in Section \ref{sec:experiment_auc} to the ImageNet1000 dataset. We compare \emph{faithfulness} (IncAUC and ExAUC) and \emph{globalness}. The results are consistent with Figure \ref{fig:auc}.}
  \label{fig:wg_imagenet1k}
\end{figure*}

\subsection{Isometry Examples} \label{app:additional_isometry_figure}

\textbf{Feature Selection.} 
Consider a feature selection explainer for two features (Figure \ref{fig:d_importance}A).
Both distributions indicate two explanations selected with equal probability. In the left plot, the explanations have a Hamming distance of 1, while in the right plot, the distance is 2. Since the explanations are, on average, closer together for the left explainer, the distribution should have higher globalness. 
In contrast, when explanations change but pairwise distances remain constant (i.e. isometric transformations), the globalness should not change. For example, reflecting the explanations in Figure \ref{fig:d_importance}B maintains pairwise distances between explanations. This intuition applies to feature attribution explainers as well (App. \ref{app:additional_isometry_figure} Figure \ref{fig:addnl_isometry}).

\begin{figure*}[t]
  \centering
  \includegraphics[width=0.85\textwidth]{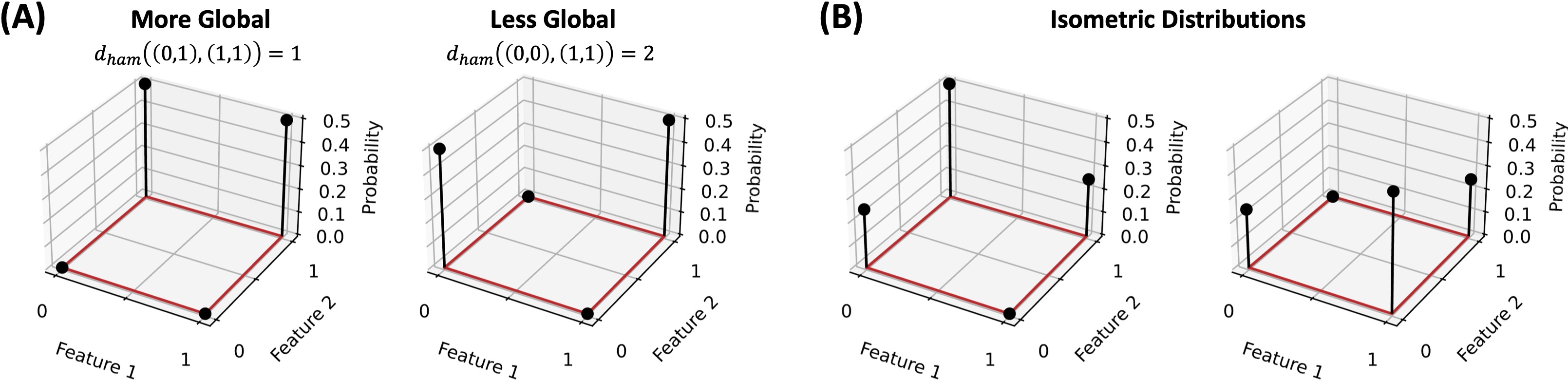}
  \vspace{-2mm}
  \caption{Toy example for a feature selection explainer. The distribution of explanations ($\mathcal{E} = \{0,1\}^2$) is shown, with probability on the vertical axis. \textbf{(A)} The left distribution is more global than the right because the average distance between explanations is lower in the metric space $(\explanationSpace, d_{\explanationSpace})$. \textbf{(B)} In contrast, distance-preserving transformations do not change the average distance between explanations. The globalness should remain the same since the distances are unchanged.}
  \label{fig:d_importance}
  \vspace{-2mm}
\end{figure*}

\textbf{Feature Attribution.} In Figure \ref{fig:addnl_isometry} we extend the example in Figure \ref{fig:example1} to visualize different isometries on 2d continuous distribution of explanations. We generate SmoothGrad attributions for classifier trained on CIFAR10 and project down to 2 dimensions for visualization. The values $\hat G_2$ shown at the top of each column show that Wasserstein Globalness is invariant to the isometric transformations shown.

  \begin{figure}[t]
  \centering
  \includegraphics[width=0.6\linewidth]{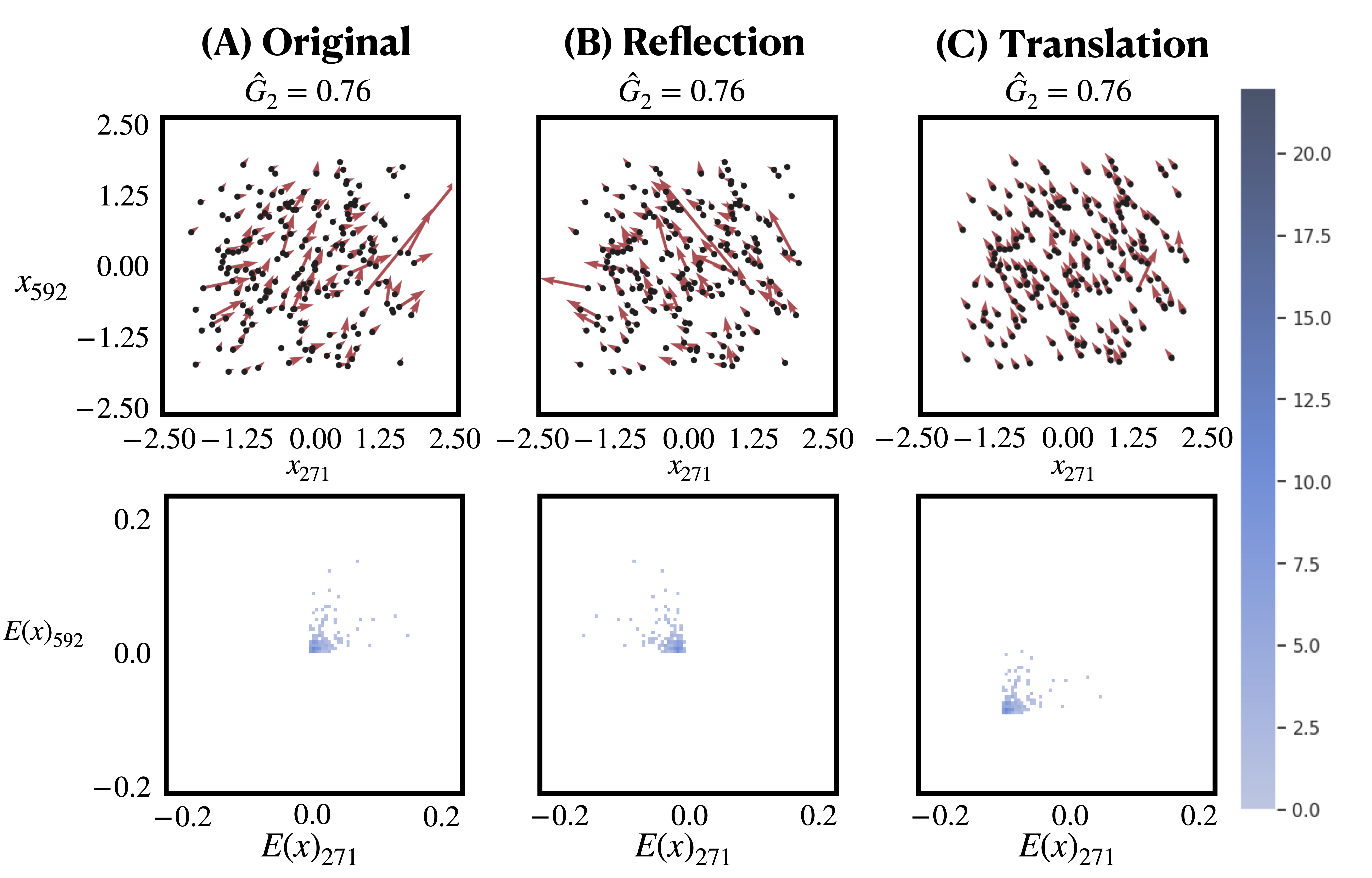}
  \caption{Example of isometries applied to SmoothGrad  attributions. Explanations are projected to two features. 
  \textbf{(Top)} Points are visualized as black dots from which corresponding attributions, shown as red vectors, originate. 
  \textbf{(Bottom)} A 2-d histogram of the respective explanations is shown. 
  \textbf{(A)} The \emph{original} explanations are shown.
  \textbf{(B)} Explanations are \emph{reflected} across the vertical axis. 
  \textbf{(C)} Explanations are \emph{translated} by the vector $\mathbf{v} = (-1e\text{-}1, 1e\text{-}1)$. 
  We observe that the isometries in \textbf{(B)} and \textbf{(C)} do not change the globalness $\hat G_2$.}
  \label{fig:addnl_isometry}
\end{figure}

\subsection{Qualitative Examples for Image Datasets} \label{app:qualitative_images}
In Figures \ref{fig:examples_mnist}, \ref{fig:examples_cifar10}, and \ref{fig:examples_imagenet} we present samples from MNIST, CIFAR10, and ImageNet200 to visualize the effects of applying explainers with higher globalness.
We observe that the WG score for the explainer increases as the smoothing parameter $\sigma$ increases. The heatmaps for $\sigma = 0$ exhibits explanations that are more distinctive for each sample, as compared to the heatmaps for $\sigma = 10$.

\begin{figure}[t]
  \centering
  \includegraphics[width=0.8\linewidth]{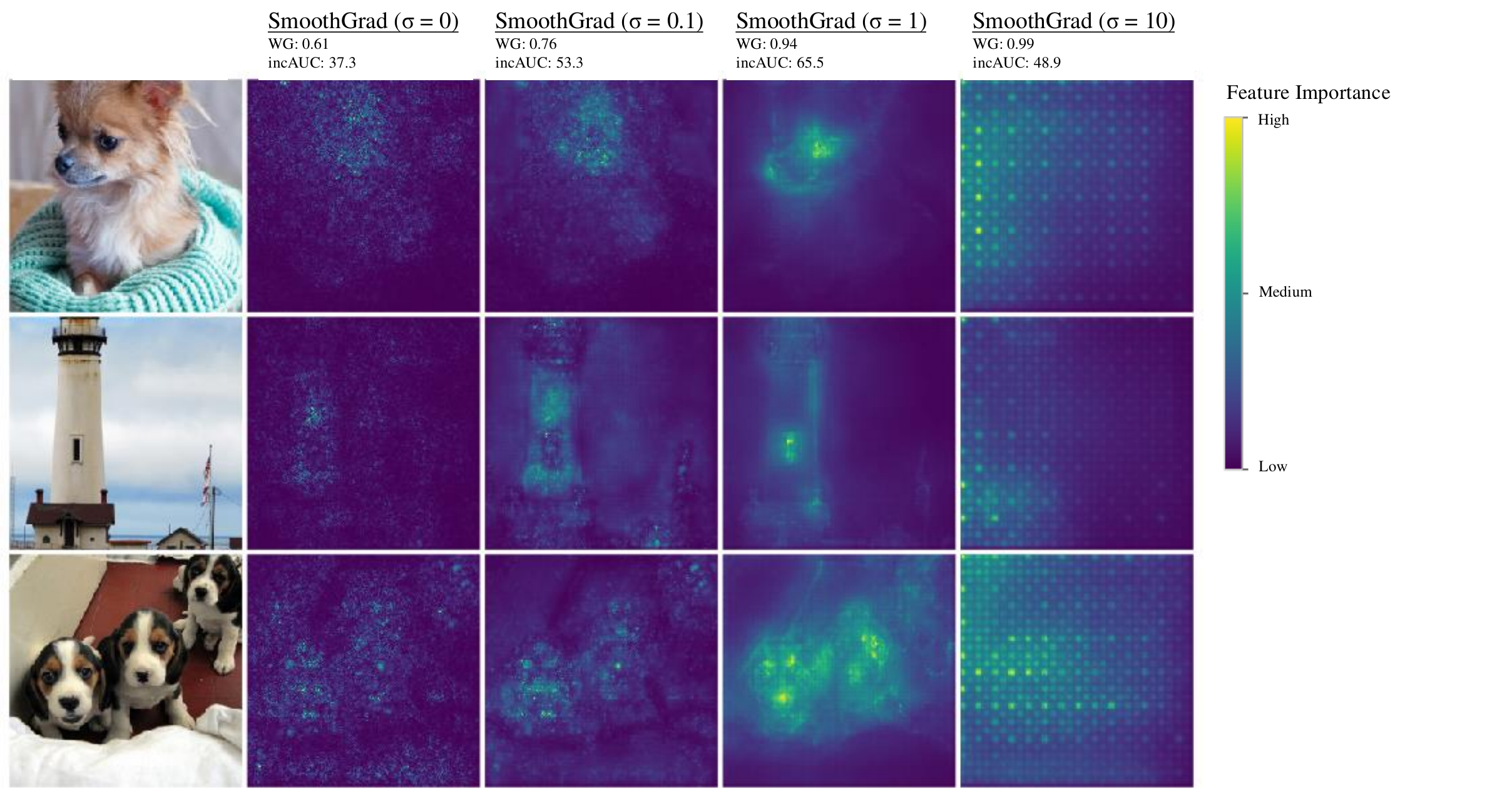}
  \caption{Heatmap of feature attribution values for samples drawn from the MNIST dataset. }
  \label{fig:examples_imagenet}
\end{figure}

\begin{figure}[t]
  \centering
  \includegraphics[width=0.8\linewidth]{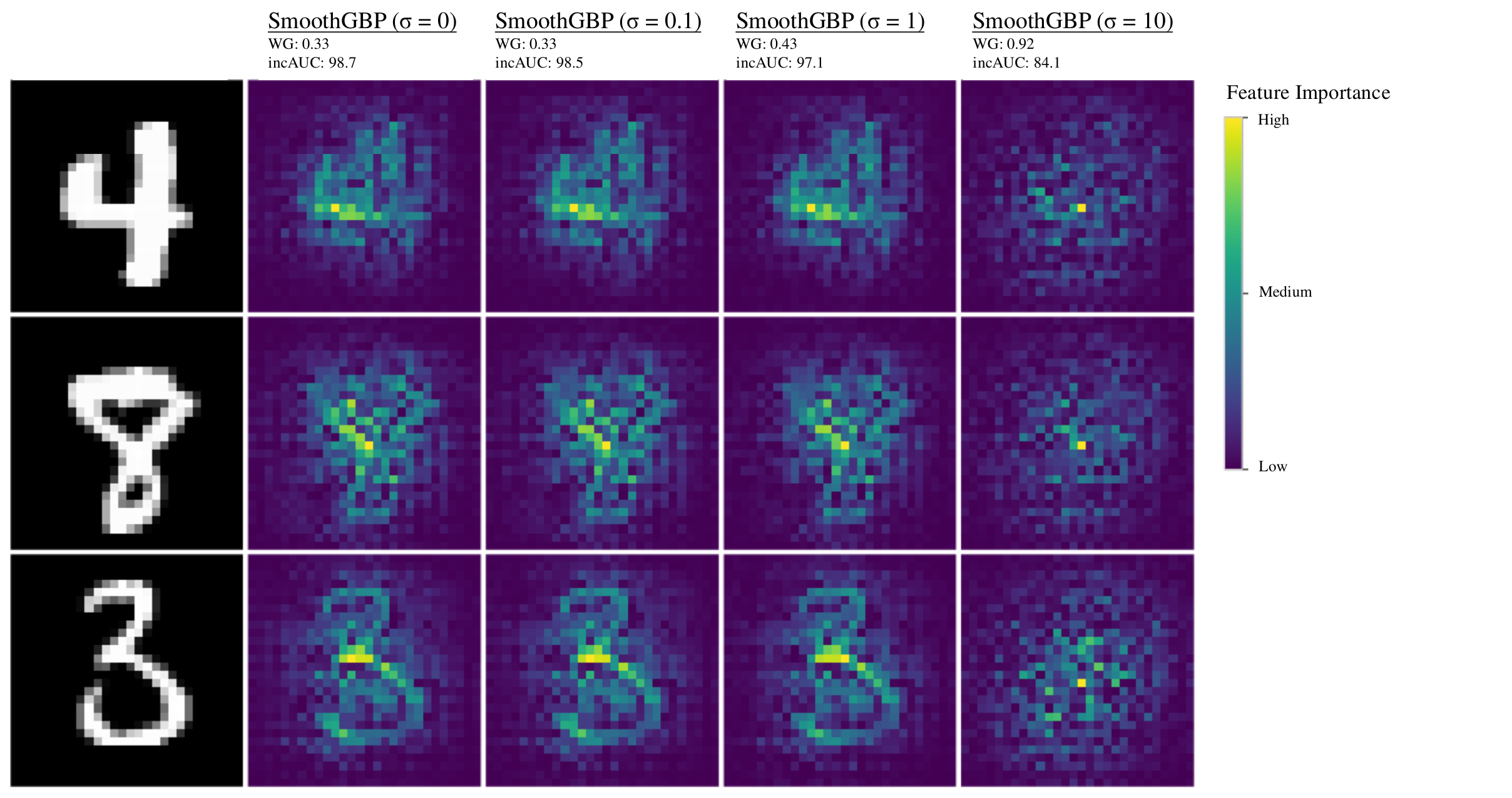}
  \caption{Heatmap of feature attribution values for samples drawn from the MNIST dataset. }
  \label{fig:examples_mnist}
\end{figure}

\begin{figure}[t]
  \centering
  \includegraphics[width=0.8\linewidth]{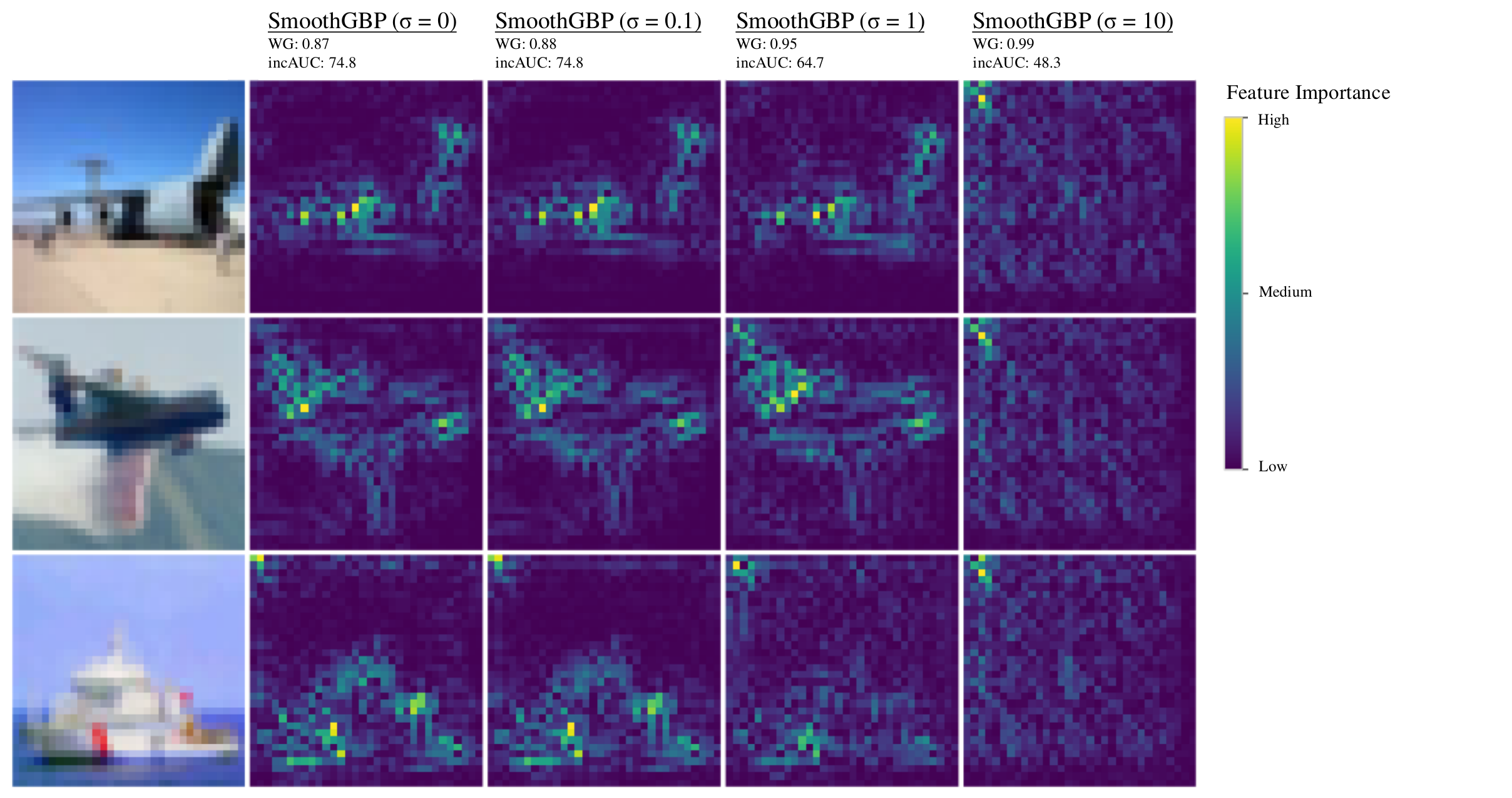}
  \caption{Heatmap of feature attribution values for samples drawn from the CIFAR10 dataset.}
  \label{fig:examples_cifar10}
\end{figure}

\subsection{Error Bars for Figure \ref{fig:auc}} \label{app:auc_sd}
In Table \ref{tab:auc_sd} we provide the bootstrap standard error results (50 iterations) corresponding to Figure \ref{fig:auc}.
\begin{table}[ht]
    \centering
    \small
    \begin{tabular}{cccccc}
    \toprule
        \textbf{Dataset} & \textbf{Method} & \textbf{$\sigma$}  & \textbf{Wasserstein Globalness SE} & \textbf{IncAUC SE} & \textbf{ExAUC SE} \\ \midrule
        Divorce & DS &   & 0.007 & 1.142 & 2.944 \\ 
        Divorce & GBP &  & 0.008 & 0.585 & 4.196 \\ 
        Divorce & GRAD &   & 0.008 & 1.41 & 2.903 \\ 
        Divorce & IG &   & 0.011 & 1.146 & 2.927 \\ 
        Divorce & SAGE &   & 0.000 & 1.844 & 2.657 \\ 
        Divorce & SBP & 0.1  & 0.008 & 0.624 & 4.187 \\ 
        Divorce & SBP & 1  & 0.009 & 0.531 & 4.221 \\ 
        Divorce & SBP & 10  & 0.01 & 0.395 & 3.852 \\ 
        Divorce & SG & 0.1  & 0.01 & 1.319 & 2.936 \\ 
        Divorce & SG & 1  & 0.009 & 1.184 & 2.938 \\ 
        Divorce & SG & 10  & 0.01 & 1.515 & 2.637 \\ 
        \midrule
        NHANES & DS &   & 0.014 & 0.062 & 0.027 \\ 
        NHANES & GBP &   & 0.014 & 0.091 & 0.021 \\ 
        NHANES & GRAD &   & 0.011 & 0.062 & 0.027 \\ 
        NHANES & IG &   & 0.003 & 0.106 & 0.017 \\ 
        NHANES & SAGE &   & 0.000 & 0.200 & 0.016 \\ 
        NHANES & SBP & 0.1  & 0.007 & 0.097 & 0.045 \\ 
        NHANES & SBP & 1 & 0.002 & 0.025 & 0.004 \\ 
        NHANES & SBP & 10  & 0.054 & 0.076 & 0.012 \\ 
        NHANES & SG & 0.1  & 0.009 & 0.148 & 0.035 \\ 
        NHANES & SG & 1  & 0.004 & 0.070 & 0.000 \\ 
        NHANES & SG & 10  & 0.004 & 0.020 & 0.009 \\ 
        \midrule
        MNIST & DS &   & 0.008 & 0.202 & 2.696 \\ 
        MNIST & GBP &   & 0.007 & 0.204 & 2.304 \\ 
        MNIST & GRAD &   & 0.006 & 0.756 & 1.854 \\ 
        MNIST & IG &   & 0.007 & 0.18 & 2.745 \\ 
        MNIST & SAGE &   & 0.000 & 1.386 & 3.9 \\ 
        MNIST & SBP & 0.1  & 0.009 & 0.309 & 2.195 \\ 
        MNIST & SBP & 1 & 0.007 & 0.769 & 1.97 \\ 
        MNIST & SBP & 10  & 0.007 & 0.905 & 2.096 \\ 
        MNIST & SG & 0.1  & 0.005 & 0.58 & 1.96 \\ 
        MNIST & SG & 1  & 0.007 & 0.497 & 2.142 \\ 
        MNIST & SG & 10  & 0.007 & 0.711 & 2.222 \\ 
        \midrule
        CIFAR10 & DS &  & 0.005 & 2.994 & 2.624 \\ 
        CIFAR10 & GBP & & 0.005 & 2.67 & 2.687 \\ 
        CIFAR10 & GRAD &  & 0.004 & 3.186 & 2.785 \\ 
        CIFAR10 & IG &  & 0.004 & 2.554 & 2.28 \\ 
        CIFAR10 & SAGE &   & 0.000 & 3.485 & 2.572 \\ 
        CIFAR10 & SBP & 0.1  & 0.004 & 2.751 & 2.173 \\ 
        CIFAR10 & SBP & 1  & 0.006 & 2.054 & 2.359 \\ 
        CIFAR10 & SBP & 10  & 0.006 & 2.09 & 2.228 \\ 
        CIFAR10 & SG & 0.1  & 0.005 & 2.415 & 2.285 \\ 
        CIFAR10 & SG & 1  & 0.006 & 2.084 & 2.318 \\ 
        CIFAR10 & SG & 10  & 0.006 & 2.179 & 2.069 \\ 
        \bottomrule
    \end{tabular}
    \caption{Bootstrap Standard Error (SE) results for Figure \ref{fig:auc} (50 iterations).}
    \label{tab:auc_sd}
\end{table}

\end{document}